\documentclass{article}

\usepackage{arxiv}

\usepackage[utf8]{inputenc} 
\usepackage[T1]{fontenc}    
\usepackage[hidelinks]{hyperref}       
\usepackage{url}            
\usepackage{booktabs}       
\usepackage{amsfonts}       
\usepackage{nicefrac}       
\usepackage{microtype}      
\usepackage{lipsum}		
\usepackage{graphicx}
\usepackage{color}
\usepackage{amsmath}
\usepackage{amssymb}
\usepackage{amsthm}
\usepackage{algorithm}
\usepackage{algpseudocode}
\usepackage{natbib}
\setcitestyle{authoryear,open={(},close={)}} 

\usepackage{doi}
\usepackage{tcolorbox}
\usepackage{tikz}
\usepackage[export]{adjustbox}
\usepackage{rotating}

\newtheorem{theorem}{Theorem}[section]
\newtheorem{proposition}[theorem]{Proposition}
\newtheorem{lemma}[theorem]{Lemma}
\newtheorem{remark}[theorem]{Remark}
\newtheorem{assumption}[theorem]{Assumption}

\newenvironment{proofsketch}{%
  \proof}{\endproof}

\title{Value of Information and Reward Specification in Active Inference and POMDPs}

\date{} 					

\author{ 
Ran Wei \\
	VERSES Research Lab\\
	\texttt{ran.wei@verses.ai} \\
}

\renewcommand{\shorttitle}{Value of Information and Reward Specification in Active Inference and POMDPs}

\hypersetup{
pdftitle={},
pdfsubject={},
pdfauthor={},
pdfkeywords={},
}

\begin{document}
\maketitle
\vspace{-0.3in}

\begin{abstract}
Expected free energy (EFE) is a central quantity in active inference which has recently gained popularity due to its intuitive decomposition of the expected value of control into a pragmatic and an epistemic component. While numerous conjectures have been made to justify EFE as a decision making objective function, the most widely accepted is still its intuitiveness and resemblance to variational free energy in approximate Bayesian inference. In this work, we take a bottom up approach and ask: \emph{taking EFE as given, what's the resulting agent's optimality gap compared with a reward-driven reinforcement learning (RL) agent, which is well understood?} By casting EFE under a particular class of belief MDP and using analysis tools from RL theory, we show that EFE approximates the Bayes optimal RL policy via information value. We discuss the implications for objective specification of active inference agents.
\end{abstract}


\section{Introduction}
Active inference \citep{parr2022active} is an agent modeling framework derived from the free energy principle, which roughly states that all cognitive behavior of an agent can be described as minimizing free energy, an information theoretic measure of the "fit" between the environment and the agent's internal model thereof \citep{friston2010free}. In recent years, active inference has seen increased popularity in various fields including but not limited to cognitive and neural science, machine learning, and robotics \citep{smith2021recent, mazzaglia2022free, lanillos2021active}. One common application of active inference across these fields is in modeling decision making behavior, often taking place in partially observable Markov decision processes (POMDP). This offers active inference as complementary, a potential alternative to, or a possible generalization of optimal control and reinforcement learning (RL). 

The central difference between active inference and RL is that instead of choosing actions that maximize expected reward or utility, active inference agents are mandated to minimize expected free energy (EFE), which in its most common form is written as \citep{da2020active}:
\begin{align}\label{eq:efe_short}
    EFE(\mathbf{a}) = -\underbrace{\mathbb{E}_{Q(\mathbf{o}|\mathbf{a})}[\log \tilde{P}(\mathbf{o})]}_{\text{Pragmatic value}} - \underbrace{\mathbb{E}_{Q(\mathbf{o}|\mathbf{a})}[\mathbb{KL}[Q(\mathbf{s}|\mathbf{o}, \mathbf{a}) || Q(\mathbf{s}|\mathbf{a})]]}_{\text{Epistemic value}}
\end{align}
Here, $\mathbf{a}$ is a sequence of actions to be evaluated, $Q(\mathbf{s}|\mathbf{a})$ and $Q(\mathbf{o}|\mathbf{a})$ are the agent's prediction of future states $\mathbf{s}$ and observations $\mathbf{o}$, $Q(\mathbf{s}|\mathbf{o}, \mathbf{a})$ is the future updated beliefs about states given future observations, $\tilde{P}(\mathbf{o})$ is a distribution encoding the agent's preferred observations and $\mathbb{KL}$ denotes Kullback-Leibler (KL) divergence, a measure of distance between two distributions. 

One can obtain an intuitive understanding of the EFE objective by analyzing the two terms separately. The first term is the negative expected log likelihood of predicted future observations under the preference or target distribution, which is equivalent to the cross entropy between the predicted and preferred observation distributions. Minimizing this term encourages the agent to take actions that lead to preferred observations. It is thus usually referred to as the "pragmatic value" or "expected value". The second term is the expected KL divergence between the predicted future states and updated beliefs about future states given future observations, which quantifies the belief update amount. This term is usually referred to as "epistemic value" or "expected information gain" because it encourages the agent to take actions that lead to a higher amount of belief update -- an implicit resolution of uncertainty. 

The intuitive addition of pragmatic and epistemic values has been taken as one of the major appeals of EFE. In some sense, it puts both values under the "same currency" when evaluating the total value of actions \citep{friston2015active}. This perspective has motivated prior work to interpret epistemic value as the "value of information" \citep{da2020active}, a term which has a similar connotation in economics \citep{howard1966information}.
Indeed, experimental evaluations of active inference agents have shown that the epistemic value term in EFE contributes to structured exploratory behavior, resolving uncertainty before attempting to obtain reward, often leading to higher coverage of the state space and enhanced task performance \citep{millidge2020deep, tschantz2020scaling, engstrom2024resolving}. Such a behavior primitive is especially important in challenging partially observable task environments.

It appears, at a first glance, that RL and optimal control miss the epistemic value term. However, it is widely known that the Bayes optimal policy in POMDPs already trades off exploration and exploitation \citep{roy2005finding}. This makes intuitive sense because resolving uncertainty often leads to more downstream rewards, essentially by "opening up" opportunities. Specifically, the Bayes optimal policy leverages the equivalence between POMDPs and a special class of MDPs defined on the reward and transition of beliefs called \emph{belief MDPs} to characterize the expected value (i.e., cumulative reward) following an action given the current belief, from which an optimal policy can be constructed as a mapping from beliefs to actions \citep{kaelbling1998planning}. These policies, as demonstrated by Bayes adaptive RL and meta RL, also exhibit structured exploratory behavior \citep{zintgraf2019varibad, duan2016rl}. It thus begs the question:

\emph{What is the relationship between the Bayes optimal RL policy and the active inference policy based on optimizing EFE?}

The main contribution of this paper is providing one answer to the above question: \vspace{0.05in}\\ 
\vspace{0.05in}\centerline{\emph{EFE approximates the Bayes optimal RL policy via epistemic value.}}
We achieve this by first establishing the equivalence between the EFE objective and a different class of belief MDPs, which allows us to define EFE-optimal policies rather than action sequences (i.e., plans) to form direct comparisons with RL policies. We then examine the source of epistemic behavior in POMDPs using a definition of the value of information for POMDPs based on \citeauthor{howard1966information}'s information value theory (\citeyear{howard1966information}). In brief, the value of information is the difference in the expected values between the Bayes optimal policy and another "naive" policy which plans as if it would not be able to update beliefs based on observations in the future. When casting the latter policy also using belief MDPs, we observe that it uses the same belief transition dynamics as the EFE policy but it uses the same belief reward as the Bayes optimal policy. Our key result is a regret bound showing that the EFE objective closes the performance gap between the naive policy and the Bayes optimal policy by augmenting the reward function of the former with epistemic value. We discuss the implications of our results for specifying active inference agents in practice.

Our work is complementary to prior work examining the relationship between active inference and RL \citep{millidge2020relationship, watson2020active, da2023reward} and the effect of epistemic value on agent behavior \citep{schwobel2018active, koudahl2021epistemics}. However, instead of trying to derive the EFE objective from first principles, we take a "bottom up" approach and analyze it against the well-known Bayes optimal policy. To our knowledge, this is the first regret bound of active inference agents in reward seeking tasks. 

\section{Background}
In this section, we introduce notations for Markov decision process, partially observable Markov decision process, and the belief MDP view of POMDPs. We then introduce active inference and the EFE objective. 

\subsection{Markov Decision Process}
A discrete time infinite-horizon discounted Markov decision process (MDP; \citealp{sutton2018reinforcement}) is defined by a tuple $M = (\mathcal{S}, \mathcal{A}, P, R, \mu, \gamma)$, where $\mathcal{S}$ is a set of states, $\mathcal{A}$ a set of actions, $P: \mathcal{S} \times \mathcal{A} \rightarrow \Delta(\mathcal{S})$ a state transition probability distribution (also called transition dynamics), $R: \mathcal{S} \times \mathcal{A} \rightarrow \mathbb{R}$ a reward function, $\mu: \Delta(\mathcal{S})$ the initial state distribution, and $\gamma \in (0, 1)$ a discount factor. In this work, we consider planning as opposed to learning, where the MDP tuple $M$ is known to the agent rather than having to be estimated from samples obtained by interacting with the environment defined by $M$. We use $\pi: \mathcal{S} \rightarrow \Delta(\mathcal{A})$ to denote a time-homogeneous Markovian policy which maps a state to a distribution over actions. Rolling out a policy in the environment for a finite number of time steps $T$ induces a sequence of states and actions $\tau = (s_{0:T}, a_{0:T})$ (also  known as a trajectory) which is distributed according to:
\begin{align}\label{eq:mdp_traj}
    P(\tau) = \prod_{t=0}^{T}P(s_{t}|s_{t-1}, a_{t-1})\pi(a_{t}|s_{t})\,,
\end{align}
where $P(s_{o}|s_{-1}, a_{-1}) = \mu(s_{0})$. We use $\rho^{\pi}_{P}(s, a) = \mathbb{E}[\sum_{t=0}^{\infty}\gamma^{t}\Pr(s_{t}=s, a_{t}=a)]$ to denote the state-action occupancy measure of policy $\pi$ in environment with dynamics $P$, where the expectation is taken w.r.t. the interaction process (\ref{eq:mdp_traj}) for $T \rightarrow \infty$. We denote the normalized occupancy measure, also called the marginal state distribution or state marginal, as $d^{\pi}_{P}(s, a) = (1 - \gamma)\rho^{\pi}_{P}(s, a)$.

Solving a MDP refers to finding a policy $\pi$ which maximizes the expected cumulative discounted reward in the environment $J(\pi)$ defined as:
\begin{align}\label{eq:mdp_obj}
    J(\pi) = \mathbb{E}\left[\sum_{t=0}^{\infty}\gamma^{t}R(s_{t}, a_{t})\right]\,.
\end{align}
The process of finding an optimal policy is sometimes referred to as reinforcement learning and it is a well-known result that there exists at least one time-homogeneous Markovian policy which is optimal w.r.t. (\ref{eq:mdp_obj}) \citep{sutton2018reinforcement}. This significantly simplifies our analysis later compared to finite horizon un-discounted MDPs for which the optimal policy is time-dependent. The quantity $\frac{1}{1 - \gamma}$ has a similar notion to planning horizon, because it represents the time step at which discounting is effectively zero. The optimal policy $\pi^{*}$ is characterized by the Bellman optimality equation:
\begin{align}\label{eq:mdp_bellman}
    Q(s, a) = R(s, a) + \gamma \mathbb{E}_{s' \sim P(\cdot|s, a)}[V(s')], \quad V(s) = \max_{a}Q(s, a)\,,
\end{align}
from which it can be obtained by taking the action which maximizes the action value function $Q$ for each state as $\pi^{*}(a|s) = \delta(a - \arg\max_{\tilde{a}}Q(s, \tilde{a}))$, where $\delta(a - b)$ is the dirac delta distribution which has probability 1 if $a = b$ and probability 0 elsewhere. The advantage function $A(s, a) = Q(s, a) - V(s) \leq 0$ quantifies the suboptimality of an action. We will omit the $*$ notation in most cases. When needed, we denote the value and advantage functions associated with policy $\pi$ and MDP $M$ as $Q^{\pi}_{M}, V^{\pi}_{M}, A^{\pi}_{M}$.

\subsection{Partially Observable Markov Decision Process}
A discrete time infinite-horizon discounted partially observable MDP (POMDP; \citealp{kaelbling1998planning}) is characterized by a tuple $M = (\mathcal{S}, \mathcal{A}, \mathcal{O}, P, R, \mu, \gamma)$, where the newly introduced symbol $\mathcal{O}$ is a set of observations, and the new transition dynamics $P$ consists of the state transition probability distribution $P(s_{t+1}|s_{t}, a_{t})$ and an observation emission distribution $P(o_{t}|s_{t})$. In a POMDP environment, the agent only has access to observations emitted from the environment state but not the state itself. It is thus generally not sufficient to consider Markovian policies but policies that depend on the history of observation-action sequences, i.e., $\pi(a_{t}|h_{t})$ where $h_{t} = (o_{0:t}, a_{0:t-1})$.

It is a well-known result that the Bayesian belief distribution $b_{t} = P(s_{t}|h_{t})$ is a sufficient statistic for the interaction history \citep{kaelbling1998planning}. The history dependent value functions and policy in POMDP can thus be written in terms of beliefs:
\begin{align}\label{eq:pomdp_bellman}
    Q(b, a) = \sum_{s}b(s)R(s, a) + \gamma \sum_{o'}P(o'|b, a)V(b'(o', a, b)), \quad V(b) = \max_{a}Q(b, a)\,,
\end{align}
where $P(o'|b, a) = \sum_{s, s'}P(o'|s')P(s'|s, a)b(s)$ and $b'(o', a, b)$ denotes the belief update function from prior $b(s)$ to the posterior:
\begin{align}\label{eq:belief_update}
    b'(o', a, b) := b'(s'|o', a, b) = \frac{P(o'|s')\sum_{s}P(s'|s, a)b(s)}{\sum_{s'}P(o'|s')\sum_{s}P(s'|s, a)b(s)}\,.
\end{align}
The optimal policy derived from the above value functions is sometimes referred to as the Bayes optimal policy \citep{duff2002optimal}.

The belief value functions in (\ref{eq:pomdp_bellman}) imply a special class of MDPs known as \emph{belief MDPs} \citep{kaelbling1998planning} where the reward and dynamics are defined on the belief state as:
\begin{align}\label{eq:pomdp_belief_mdp}
    R(b, a) = \sum_{s}b(s)R(s, a), \quad P(b'|b, a) = P(o'|b, a)\delta(b' - \tilde{b}'(o', a, b))\,.
\end{align}
The stochasticity in the belief dynamics is entirely due to the stochasticity of the next observation; the belief updating process itself is deterministic.

In this work, we generalize the notion of belief MDP to refer to any MDP defined on the space of beliefs. However, not all belief MDPs could yield the optimal policies for some POMDPs.

\subsection{Active Inference}\label{sec:act_inf}
Active inference is an application of the variational principle to perception and action, where intractable Bayesian belief updates (i.e., (\ref{eq:belief_update})) are approximated by variational inference \citep{da2020active}. At every time step $t$, variational inference searches for an approximate posterior $Q(s_{t})$ which maximizes the evidence lower bound of data marginal log likelihood, or equivalently minimizes the variational free energy $\mathcal{F}$:
\begin{align}
    \mathcal{F}(Q) = \mathbb{E}_{Q(s_{t})}[\log Q(s_{t}) - \log P(o_{t}, s_{t})] \,,
\end{align}
where $P(o_{t}, s_{t}) = P(o_{t}|s_{t})P(s_{t})$. In the context of POMDPs, the prior is given by $P(s_{t}) = \sum_{s_{t-1}}P(s_{t}|s_{t-1}, a_{t-1})Q(s_{t-1})$. It is well-known that the optimal variational approximation under appropriately chosen family of posterior distributions equals to the exact posterior in (\ref{eq:belief_update}) \citep{blei2017variational}. We will thus assume appropriate choices of variational family and omit suboptimal belief updating in subsequent analyses.

Central to the current discussion is the policy selection objective functions used in active inference, which is its main difference from classic POMDPs. In particular, active inference introduces an objective function called expected free energy (EFE) which, given an initial belief $Q_{0}(s_{0})$ and a finite sequence of actions $a_{0:T-1}$, is defined as \citep{friston2017active}:
\begin{align}\label{eq:efe_full}
    EFE(a_{0:T-1}, Q_{0}) = \mathbb{E}_{Q(o_{1:T}, s_{1:T}|a_{0:T-1})}[\log Q(s_{1:T}|a_{0:T-1}) - \log \Tilde{P}(o_{1:T}, s_{1:T})] \,,
\end{align}
where $Q(s_{1:T}|a_{0:T-1})$ is defined as the product of the \emph{marginal} state distributions along the action sequence (we show how this can be approximately obtained as a result of variational inference and discuss the implication of defining this instead as the \emph{joint} distribution in the appendix, which also contains all derivations and proofs):
\begin{align}\label{eq:efe_q_def}
\begin{split}
    Q(s_{1:T}|a_{1:T-1}) &= \prod_{t=1}^{T}Q(s_{t}|Q_{t-1}, a_{t-1}) \,, \\ 
    Q(s_{t}|Q_{t-1}, a_{t-1}) &:= \sum_{s_{t-1}}P(s_{t}|s_{t-1}, a_{t-1})Q(s_{t-1}|Q_{t-2}, a_{t-2}) \,,
\end{split}
\end{align}
and $Q(o_{1:T}, s_{1:T}|a_{1:T-1}) = \prod_{t=1}^{T}P(o_{t}|s_{t})Q(s_{t}|Q_{t-1}, a_{t-1})$. These distributions represent the agent's posterior predictive beliefs about states and observations in the future. Notice (\ref{eq:efe_full}) is different from (\ref{eq:efe_short}), but it is used here because it is more general \citep{champion2024reframing}. 

The distribution $\tilde{P}(o_{1:T}, s_{1:T})$ is interpreted as a "preference" distribution under which preferred observations and states have higher probabilities. While there are multiple ways to specify $\tilde{P}$ in the literature, we will focus on the most popular specification:
\begin{align}
    \tilde{P}(o_{0:T}, s_{0:T}) &= \prod_{t=0}^{T}\tilde{P}(o_{t})\tilde{P}(s_{t}|o_{t}) \,,
\end{align}
where $\tilde{P}(s_{t}|o_{t})$ is an arbitrary distribution. This specification allows us to factorize EFE over time and construct the following approximation:
\begin{align}\label{eq:efe_prag_epis}
    EFE(a_{0:T-1}, Q_{0}) \approx \sum_{t=1}^{T}-\underbrace{\mathbb{E}_{Q(o_{t}|a_{0:T-1})}[\log\tilde{P}(o_{t})]}_{\text{Pragmatic value}} - \underbrace{\mathbb{E}_{Q(o_{t}|a_{0:T-1})}[\mathbb{KL}[Q(s_{t}|o_{t}, a_{0:T-1}) || Q(s_{t}|a_{0:T-1})]]}_{\text{Epistemic value}} \,,
\end{align}
where $\mathbb{KL}$ denotes Kullback-Leiblier divergence, $Q(o_{t}|a_{0:T-1}) = \sum_{s_{t}}P(o_{t}|s_{t})Q(s_{t}|Q_{t-1}, a_{t-1})$ is the posterior predictive over observations, and $Q(s_{t}|o_{t}, a_{0:T-1}) \propto P(o_{t}|s_{t})Q(s_{t}|Q_{t-1}, a_{t-1})$ is the future posterior given posterior predictive of future states as prior and future observations. We discuss this approximation and optimal choice of $\tilde{P}(s|o)$ further in the appendix.

As preempted in the introduction, in (\ref{eq:efe_prag_epis}), the first term "pragmatic value" scores the quality of predicted observations under the preferred distribution. The second term "epistemic value" measures the distance between future prior $Q(s_{t}|a_{0:T-1})$ and posterior beliefs $Q(s_{t}|o_{t}, a_{0:T-1})$, which corresponds to the amount of expected "information gain" from future observations. The epistemic value term is an especially salient difference between active inference and classic POMDPs.

\section{Unifying Active Inference and RL Under Belief MDPs}\label{sec:act_inf_belief_mdp}
The use of EFE vs. reward and the search for action sequences (i.e., plans) vs. policies are the main contentions between active inference and RL. In this section, we show that active inference can be equally represented using reward and policy in a special class of belief MDPs. The key is to show that the EFE objective can be characterized using a recursive equation akin to the Bellman equation. This can be achieved immediately by expressing the predictive distribution at each step using the predictive distribution at the previous step:
\begin{align}
\begin{split}
    &EFE(a_{0:T-1}, Q_{0}) \\
    &\approx \sum_{t=1}^{T}-\mathbb{E}_{Q(o_{t}|a_{0:T-1})}[\log\tilde{P}(o_{t})] - \mathbb{E}_{Q(o_{t}|a_{0:T-1})}[\mathbb{KL}[Q(s_{t}|o_{t}, a_{0:T-1}) || Q(s_{t}|a_{0:T-1})]] \\
    &= \sum_{t=0}^{T-1}-\mathbb{E}_{Q(o_{t+1}|Q_{t}, a_{t})}[\log\tilde{P}(o_{t+1})] - \mathbb{E}_{Q(o_{t+1}|Q_{t}, a_{t})}[\mathbb{KL}[Q(s_{t+1}|o_{t+1}, Q_{t}, a_{t}) || Q(s_{t+1}|Q_{t}, a_{t})]] \\
    &= -\mathbb{E}_{Q(o_{1}|Q_{0}, a_{0})}[\log\tilde{P}(o_{1})] - \mathbb{E}_{Q(o_{1}|Q_{0}, a_{0})}[\mathbb{KL}[Q(s_{1}|o_{1}, Q_{0}, a_{0}) || Q(s_{1}|Q_{0}, a_{0})]] + EFE(a_{1:T-1}, Q_{1}) \,.
\end{split}
\end{align}

The recursive equation implies a transition dynamics over the state marginal $Q_{t}$ which only depends on the previous state marginal $Q_{t-1}$, i.e., the transition is Markovian. The per-time step EFE only depends on the current state marginal. Using the equivalence between the optimal $Q_t$ and $b_t$, we can write the reward and transition dynamics of the belief MDP implied by EFE as follows:
\begin{subequations}\label{eq:efe_belief_mdp}
\begin{align}
    R^{EFE}(b, a) &= \mathbb{E}_{P(o'|b, a)}[\log\tilde{P}(o')] + \mathbb{E}_{P(o'|b, a)}[\mathbb{KL}[b(s'|o', b, a) || b(s'|b, a)]] \\
    &:= \tilde{R}(b, a) + IG(b, a) \label{eq:efe_belief_mdp_rwd} \,,\\
    P^{open}(b'|b, a) &= \delta(b' - b'(a, b)), \text{ where } b'(a, b) := b'(s'|b, a) = \sum_{s}P(s'|s, a)b(s) \label{eq:efe_belief_mdp_dynamics} \,.
\end{align}
\end{subequations}

By constructing the above belief MDP, the search for optimal action sequences can be equally represented as the search for optimal belief-action policies.

\begin{proposition}\label{prop:act_inf_policy}
(Active inference policy) The EFE achieved by the optimal action sequence can be equivalently achieved by a time-indexed belief-action policy $\pi(a_{t}|b_{t})$.
\end{proposition}
\begin{proof}
The proof is due to the above belief MDP characterization. An alternative proof is given in the appendix.
\end{proof}

These identities enable us to define infinite-horizon discounted belief MDPs using the EFE reward and dynamics in (\ref{eq:efe_belief_mdp}) and restrict our search to time-homogeneous Markovian belief-action policies. A similar result was presented recently by \citet{malekzadeh2022active}. However, rather than focusing on policy optimization algorithms, our goal here is to clarify the belief MDP implied by EFE. 

However, notice a few differences between the EFE belief MDP and the Bayes optimal belief MDP. First, the belief dynamics in (\ref{eq:efe_belief_mdp_dynamics}) does not contain observation $o$; rather it is the marginal prediction of the next state given the previous belief. Such a belief dynamics has been referred to as \emph{open-loop} in the literature \citep{flaspohler2020belief} in the sense that it does not take into account the possibility of updating beliefs based on future observations, akin to open-loop controls. In contrast to the POMDP belief dynamics in (\ref{eq:pomdp_belief_mdp}), the open-loop belief dynamics is deterministic given $a$. 

Second, the EFE reward function contains an information gain term which corresponds to epistemic value. The first term pragmatic value is defined as the expected log likelihood of the next observation. This does not introduce much difference from the POMDP reward function because we can define the active inference preference distribution as a Boltzmann distribution parameterized by a reward function
$\tilde{P}(o) \propto \exp(\tilde{R}(o))$ and assume that $\tilde{R}(o)$ self-normalizes so that the partition function equals 1. The resulting reward can still be written as a linear combination of state-action reward:
\begin{align}\label{eq:prag_value_linear}
\begin{split}
    \tilde{R}(b, a) &= \mathbb{E}_{P(o'|b, a)}[\log\tilde{P}(o')] \\
    &= \sum_{s}b(s)\sum_{s'}P(s'|s, a)\sum_{o'}P(o'|s')\tilde{R}(o') \\
    &= \sum_{s}b(s)\tilde{R}(s, a) \,.
\end{split}
\end{align}

The linearity and thus convexity of the state-action reward is an important property of POMDPs, because it implies the optimal value function is also convex in the beliefs, which means that lower entropy or more certain beliefs generally correspond to higher values \citep{kaelbling1998planning}. The addition of information gain, however, makes the EFE reward no longer convex. In this case, the agent may be driven to collect more information and "distracted" from accruing task rewards.

\begin{proposition}\label{eq:prop_efe_reward_concavity}
The EFE reward function as defined in (\ref{eq:efe_belief_mdp_rwd}) is concave in the belief.
\end{proposition}
\begin{proofsketch}
Information gain can be rearranged as follows:
\begin{align}
    IG(b, a) = \mathbb{H}[P(o'|b, a)] - \mathbb{E}_{P(s'|b, a)}[\mathbb{H}[P(o'|s')]]
\end{align}
where $\mathbb{H}$ denotes Shannon entropy and the second term is linear in the belief. Since entropy is concave, the combined reward function is also concave. 
\end{proofsketch}

In summary, the EFE objective and the classic POMDP can be understood as two different belief MDPs with different reward functions and different dynamics.

\section{Analyzing Policies in MDPs}\label{sec:pdmmdp}
The belief MDP characterizations of both the EFE policy and the Bayes optimal policy enable us to use MDP analysis tools for POMDPs. The main analysis tools we use in this paper are recent extensions of the performance difference lemma \citep{kakade2002approximately} and simulation lemma \citep{kearns2002near} which are well-known results in RL theory that quantify the performance difference between different policies or the same policy in different environments. To compare active inference with RL, we are interested in the setting where two policies are optimal w.r.t. both different rewards and different dynamics, however, the evaluation reward and dynamics are equivalent to only one of the policies (here referred to as the expert policy). The following lemma, which extends lemma 4.1 in \citep{vemula2023virtues} to the setting of different rewards, gives the performance gap (also known as the regret) between the two policies:
\begin{lemma}\label{lemma:pdmmdp}
(Performance difference in mismatched MDPs) Let $\pi$ and $\pi'$ be two policies which are optimal w.r.t. two MDPs $M$ and $M'$. The two MDPs share the same initial state distribution and discount factor but have different rewards $R, R'$ and dynamics $P, P'$. Denote $\Delta R(s, a) = R'(s, a) - R(s, a)$. The performance difference between $\pi$ and $\pi'$ when both are evaluated in $M$ is given by:
\begin{align}\label{eq:pdmmdp}
\begin{split}
    &J_{M}(\pi) - J_{M}(\pi') \\
    &= \underbrace{\frac{1}{(1 - \gamma)}\mathbb{E}_{(s, a) \sim d^{\pi}_{P}}\left[A^{\pi'}_{M'}(s, a)\right]}_{\text{Policy advantage under expert distribution}} \\ &\quad + \underbrace{\frac{1}{(1 - \gamma)}\mathbb{E}_{(s, a) \sim d^{\pi'}_{P}}\left[\Delta R(s, a) + \gamma\left(\mathbb{E}_{s' \sim P'(\cdot|s, a)}[V^{\pi'}_{M'}(s')] - \mathbb{E}_{s'' \sim P(\cdot|s, a)}[V^{\pi'}_{M'}(s'')]\right)\right]}_{\text{Reward-model advantage under own distribution}} \\
    &\quad + \underbrace{\frac{1}{(1 - \gamma)}\mathbb{E}_{(s, a) \sim d^{\pi}_{P}}\left[-\Delta R(s, a) + \gamma\left(\mathbb{E}_{s'' \sim P(\cdot|s, a)}[V^{\pi'}_{M'}(s'')] - \mathbb{E}_{s' \sim P'(\cdot|s, a)}[V^{\pi'}_{M'}(s')]\right)\right]}_{\text{Reward-model disadvantage under expert distribution}} \,. \\
\end{split}
\end{align}
\end{lemma}

Lemma \ref{lemma:pdmmdp} decomposes the performance gap in MDP $M$ between policy $\pi$ (the expert) and $\pi'$ into three terms. The first term is the advantage value of $\pi'$ under the expert's state-action marginal distribution. The second term is the difference in reward between MDP $M'$ and $M$ and the difference in the value $V^{\pi'}_{M'}$ of $\pi'$ in $M'$ due to the difference in dynamics expected under the state-action marginal distribution of $\pi'$. This term quantifies the "advantage" of being evaluated in one MDP vs another. The last term is the opposite of reward-model advantage, i.e., disadvantage, expected under the expert policy $\pi$'s state-action marginal distribution.

One can obtain an intuitive understanding of (\ref{eq:pdmmdp}) by attempting to minimize the performance gap via optimizing $R', P'$, given we require $\pi'$ to be the optimal policy w.r.t. some $R', P'$. First, it holds that when $R', P'$ are respectively equal to $R, P$, the reward and model advantages are zeros, and the policy advantage is zero as a result. This means one can read policy, reward, and model advantage as a measure of error from the expert MDP and policy. When such error is nonzero, $R', P'$ are optimized to increase reward-model advantage under the expert distribution and decrease reward-model advantage under the policy's own distribution. This encourages $\pi'$ to choose actions that lead to state-actions achieved by the expert policy, eventually matching expert distribution and thus expert performance. This property has been used to learn value-aware dynamics models to robustly imitate expert behavior in offline inverse reinforcement learning \citep{wei2023bayesian}.

Using (\ref{eq:pdmmdp}), we can obtain an upper bound on the performance gap in terms of the policy advantage and reward and model difference:

\begin{lemma}\label{lemma:pdmmdp_bound}
For the setting considered in lemma \ref{lemma:pdmmdp}, let $\epsilon_{\pi'} = \mathbb{E}_{(s, a) \sim d^{\pi}_{P}}[|A^{\pi'}_{M'}(s, a)|]$, $\epsilon_{R'} = \mathbb{E}_{(s, a) \sim d^{\pi}_{P}}[|\Delta R(s, a)|]$, $\epsilon_{P'} = \mathbb{E}_{(s, a) \sim d^{\pi}_{P}}[\mathbb{KL}[P(\cdot|s, a) || P'(\cdot|s, a)]]$, and $R'_{max} = \max_{s, a}|R'(s, a)|$. Let the two policies have bounded state-action marginal density ratio $\frac{d^{\pi'}_{P}(s, a)}{d^{\pi}_{P}(s, a)} \leq C$. The performance gap is bounded as:
\begin{align}
    J_{M}(\pi) - J_{M}(\pi') \leq \frac{1}{1 - \gamma}\epsilon_{\pi'} + \frac{C + 1}{1 - \gamma }\epsilon_{R'} + \frac{(C + 1)\gamma R'_{max}}{(1 - \gamma)^{2}}\sqrt{2\epsilon_{P'}}
\end{align}
\end{lemma}

Lemma \ref{lemma:pdmmdp_bound} shows that the performance gap is linear (w.r.t. planning horizon) in the expected policy advantage and reward difference and quadratic in the model difference. Thus, model difference is a main contributor to performance difference if it has a similar magnitude to policy and reward differences. However, it should be noted that this bound can be overly conservative (sometimes known as the worst-case bound; \citealp{ross2011reduction}) since it doesn't consider the possibility of reward advantage being cancelled out by model advantage. 

\section{Value of Information in POMDPs}
Given the primary difference between active inference and RL is the definition of epistemic value and open-loop belief dynamics, we ask whether it could be seen as an approximation to the Bayes optimal policy, specifically the epistemic aspect thereof? To this end, we first analyze the "value of information" in the Bayes optimal policy. We then show that epistemic value closes the gap to the Bayes optimal policy by making up for the loss of information value.

\subsection{Value of Information in Bayes Optimal RL Policy}\label{sec:voi_rl}
It's colloquially accepted that the Bayes optimal policy characterized by the value functions in (\ref{eq:pomdp_bellman}) optimally trades off exploration and exploitation. However, it's not immediately obvious what is being traded off, the comparison is made against which alternative action or policy, and how large is the performance gap. 
In this paper, we adopt the view that what's being traded off is the value of information, which we try to quantify in an action or policy.
In \citep{howard1966information}, the value of information for a single step decision making problem is defined as the reward a decision maker is willing to give away if they could have their uncertainty resolved (e.g., by a clairvoyant). Formally, the expected value of perfect information (EVPI) is defined as the difference between the expected value given perfect information (EV$|$PI) and the expected value without perfect information (EV). 

In the POMDP setting, the agent cannot in general obtain perfect information about the hidden state, but an observation that is usually correlated with the state. It turns out that this corresponds to an extension of Howard's definition in the single step decision making setting called the value of imperfect information \citep{raiffa2000applied}. For consistency in notation, we will label it as the expected value of perfect observation (EVPO) and define it as: 
\begin{align}\label{eq:epvi}
\begin{split}
    EVPO &= EV|PO - EV \,,\\
    EV &= \max_{a} \sum_{s}b(s)R(s, a) \,,\\
    EV|PO &= \sum_{o}\sum_{s}P(o|s)b(s)\max_{a}R(b(s|o), a) \,.
\end{split}
\end{align}
Similar to EVPI, EVPO is non-negative because an optimal decision maker cannot gain information and do worse \citep{howard1966information}.

Extending this definition for the multi-stage sequential decision making setting, we have the following corollary of EV and EV$|$PO for POMDPs:
\begin{subequations}\label{eq:belief_mdp_open_close}
\begin{align}
EV: &\quad Q^{open}(b, a) = \sum_{s}b(s)R(s, a) + \gamma V^{open}(b'(a, b)) \label{eq:belief_mdp_open} \,,\\
EV|PO: &\quad Q(b, a) = \sum_{s}b(s)R(s, a) + \gamma \sum_{o'}P(o'|b, a)V(b'(o, a, b)) \label{eq:belief_mdp_close} \,.
\end{align}
\end{subequations}

The definition is the same as that of \cite{flaspohler2020belief}, except here we introduce additional motivation and justification based on the framework of \citet{howard1966information} and \citet{raiffa2000applied}. It is clear that EV$|$PO is the same as the Bayes optimal value function. Interestingly, EV uses the open-loop belief dynamics that we saw earlier in EFE but it uses the same reward as Bayes optimal policy. We thus label its value functions as $Q^{open}$ and $V^{open}$. The following proposition shows that EVPO in the POMDP setting is also non-negative:

\begin{proposition}\label{prop:evpo_pomdp_nonnegative}
Let $Q^{open}(b, a), V^{open}(b)$ and $Q(b, a), V(b)$ denote the open and closed-loop value functions as defined in (\ref{eq:belief_mdp_open_close}), it holds that:
\begin{align}
    Q(b, a) \geq Q^{open}(b, a) \text{ and } V(b) \geq V^{open}(b) \text{ for all $b \in \Delta(\mathcal{S})$ and $a \in \mathcal{A}$} \,.
\end{align}
\end{proposition}

Intuitively, the closed-loop Bayes optimal policy is better because it can take actions that lead to future observations which upon update lead to lower entropy beliefs. Given the closed-loop value function is convex in the beliefs, lower entropy beliefs generally have higher value. These actions are referred to as \emph{epistemic} actions.

However, simply comparing open and closed-loop value functions doesn't give us an adequate measure of the value of information since in most realistic settings, agents are allowed to observe the environment and update their beliefs despite using potentially suboptimal open-loop policies. We thus consider this setting by deploying both policies in a POMDP for which the closed-loop policy is optimal, and the only difference between the two policies is that the open-loop policy will choose actions according to (\ref{eq:belief_mdp_open}) as if it would not be able to observe the environment in the future. From lemma \ref{lemma:pdmmdp} we know that the primary contributor to the performance gap between the two policies is the difference in their transition dynamics and the resulting model advantage. The following proposition characterizes the advantage of the closed-loop dynamics:

\begin{proposition}\label{prop:closed_loop_model_advantage}
Let $R_{max} = \max_{s, a}|R(s, a)|$. The closed-loop model advantage is bounded as follows:
\begin{align}
\begin{split}
    0 \leq \mathbb{E}_{P(b'|b, a)}[V^{open}(b')] - \mathbb{E}_{P^{open}(b''|b, a)}[V^{open}(b'')] \leq \frac{R_{max}}{1 - \gamma}\sqrt{2IG(b, a)} \,.\\
\end{split}
\end{align}
\end{proposition}

It shows that the advantage of closed-loop dynamics is primarily due to information gain which scales linearly w.r.t. the planning horizon.

\subsection{Main Result: EFE Approximates Bayes Optimal RL Policy}\label{sec:voi_efe}
The main insight of this work is that EFE closes the optimality gap between open and closed-loop policies by augmenting the reward of the open-loop policy with the epistemic value term. Given the pragmatic value is linear in the belief (\ref{eq:prag_value_linear}), we will use it as the shared reward between active inference and RL agents, i.e., $R(s, a) = \tilde{R}(s, a)$.

Proposition \ref{prop:closed_loop_model_advantage} shows that the advantage of closed-loop belief transition is proportional to the information gain provided by the next observation. While the agent cannot change either belief transition distributions, it can change its reward to alter the reward-model advantage and the marginal distribution under which it is evaluated. An obvious choice for the reward advantage is to set it to the information gain in order to cancel with the information disadvantage of open-loop belief dynamics. To ensure the agent does not get distracted by gaining information and still focus on task relevant behavior, we make the following assumption on preference distribution specification:

\begin{assumption}\label{assumption:reward_specification}
(Preference specification) The preference distribution or reward is specified such that the gain in pragmatic value after receiving a new observation is higher than the loss in epistemic value in expectation under the Bayes optimal policy $\pi$ in closed-loop belief dynamics $P$:
\begin{align}
\begin{split}
    \mathbb{E}_{(b, a) \sim d^{\pi}_{P}}\left[\sum_{s}\left(b(s|o) - b(s)\right)R(s, a)\right] \geq \mathbb{E}_{(b, a) \sim d^{\pi}_{P}}[IG(b(s), a) - IG(b(s|o), a)] \,.
\end{split}
\end{align}
\end{assumption}

This assumption also ensures that the advantage of closed-loop belief dynamics under the EFE value function is non-negative. In practice, since the Bayes optimal policy behavior can be difficult to know a priori, we can approximate the above by setting a reward function such that the reward difference is sufficiently high. In the appendix, we prove that the advantage upper bound given this assumption is the same as that evaluated under the open-loop belief MDP in proposition \ref{prop:closed_loop_model_advantage}. To facilitate the comparison between open-loop and EFE policy, we introduce two more assumptions:

\begin{assumption}\label{assumption:policy_behavior}
(Policy behavior) We make the following assumptions on the behavior of the evaluated policies:
\begin{enumerate}
    \item The absolute advantage of the EFE policy $\pi^{EFE}$ expected under the Bayes optimal policy's marginal distribution is no worse than that of the open-loop policy $\pi^{open}$: $\epsilon_{\tilde{\pi}} = \mathbb{E}_{(b, a) \sim d^{\pi}_{P}}[|A^{\pi^{open}}_{P}(b, a)|] \geq \mathbb{E}_{(b, a) \sim d^{\pi}_{P}}[|A^{\pi^{EFE}}_{P}(b, a)|]$.
    \item For both the open-loop policy $\pi^{open}$ and EFE policy $\pi^{EFE}$, it always holds that $IG(b, a) \geq 2$ for any $b, a$ sampled from either their own or the expert policy's marginal distribution.
\end{enumerate}
\end{assumption}

Note that both assumptions are conservative but they will enable us to focus the comparison of both policies on their information seeking behavior. Assumption 1 is reasonable because we expect the EFE policy to be more similar to the expert than the open-loop policy given the information gain reward encourages information seeking behavior. This enables us to remove policy advantage from the comparison.  Assumption 2 is partly numerically motivated because it allows us to further upper bound the closed-loop model advantage in proposition \ref{prop:closed_loop_model_advantage} via $\sqrt{2\mathbb{KL}} \leq \mathbb{KL}$ so that the $IG$ reward bonus in EFE can be directly compared with closed-loop model advantage and subtracted from it. In practice, many POMDP environments are much more benign in that partial observability, and thus the value of information, decreases to zero in a small number of time steps \citep{liu2022partially}. In that case, the difference between open-loop, EFE, and Bayes optimal policies become very small. Thus, the setting we consider is harder or more pessimistic. 

The following theorem, which is the main result, gives the performance gap of both policies compared to the Bayes optimal policy:

\begin{theorem}\label{theorem:open_efe_performance_gap}
Let all policies be deployed in POMDP $M$ and all are allowed to update their beliefs according to $b'(o', a, b)$. Let $\epsilon_{IG} = \mathbb{E}_{(b, a) \sim d^{\pi}_{P}}[IG(b, a)]$ denotes the expected information gain under the Bayes optimal policy's belief-action marginal distribution and let the belief-action marginal induced by both open-loop and EFE policies have bounded density ratio with the Bayes optimal policy $\left\Vert \frac{d^{\tilde{\pi}}_{P}(b, a)}{d^{\pi}_{P}(b, a)} \right\Vert_{\infty} \leq C$. Under assumptions \ref{assumption:reward_specification} and \ref{assumption:policy_behavior}, the performance gap of the open-loop and EFE policies from the optimal policy are bounded as:
\begin{align}
\begin{split}
    &J_{M}(\pi) - J_{M}(\pi^{open}) \leq \frac{1}{1 - \gamma}\epsilon_{\tilde{\pi}} + \frac{(C + 1)\gamma R_{max}}{(1 - \gamma)^{2}}\epsilon_{IG} \,,\\
    &J_{M}(\pi) - J_{M}(\pi^{EFE}) \leq \frac{1}{1 - \gamma}\epsilon_{\tilde{\pi}} + \frac{(C + 1)\gamma R_{max}}{(1 - \gamma)^{2}}\epsilon_{IG} - \frac{C + 1}{1 - \gamma}\epsilon_{IG} \,.
\end{split}
\end{align}
\end{theorem}

Theorem \ref{theorem:open_efe_performance_gap} shows that the performance gap of both policies are linear (w.r.t. planning horizon) in the policy advantage and quadratic in the information gain. However, the EFE policy improves over the open-loop policy with a linear increase in information gain. As mentioned before, these bounds are conservative estimates since the information seeking propriety of the EFE policy could further reduce policy disadvantage and the IG bonus could further reduce closed-loop model advantage. 

\section{Discussions}
Our results highlight the nuanced relationship between active inference and the classic approach to POMDPs. In this section, we provide a few complementary perspectives on related POMDP approximation and extensions from the RL literature and discuss objective specification in active inference informed by our results. 

\subsection{POMDP Approximation and Extensions}
In the POMDP planning literature, there is a suite of approximation techniques to overcome the intractability of exact belief update and value function representation. The simplest ones are the maximum likelihood heuristic and QMDP heuristic which first compute the underlying MDP value function and then obtain the belief value function using either the most likely state under the current belief or a belief-weighted average \citep{littman1995learning}. These approximations leverage the fact that MDP value functions (in discrete space) are easy to compute, but they can be overly optimistic since they implicitly assume the state in the next time step will be fully observed \citep{hauskrecht2000value}. As a result, the agent does not take information gathering actions.

To address this shortcoming, there is a special set of heuristics dedicated to inducing information gathering actions \citep{roy2005finding}. These information gathering heuristics typically operate in a "dual-mode" fashion where exploitation and exploration are arbitrated by some criterion. For example, in \citet{cassandra1996acting}, the exploitation mode chooses actions based on the underlying MDP whereas the exploration mode chooses actions to minimize belief entropy in the next time step. These two modes are arbitrated by the entropy of the current belief. Complementary to dual-mode execution, \citet{flaspohler2020belief} propose to interleave open-loop with closed-loop belief dynamics when the value of information is low to speed up value function computation. In doing so, these methods alleviate the expensive belief updating operation during planning. While EFE resembles these heuristics and thus amenable to efficiency gain, it introduces an information gain term in the reward, which could be an expensive operation in itself \citep{belghazi2018mutual}.

Recently, there is a family of methods called information directed sampling (IDS) which also introduces information objectives primarily to improve Thompson sampling-based algorithms in the context of multi-arm bandits and Bayesian RL \citep{russo2018learning, lu2023reinforcement, hao2022regret, chakraborty2023steering}. These problems can be seen as subsets of POMDPs where the only hidden state is the unknown environment model parameters \citep{doshi2016hidden}. Similar to our work, their analyses are also based on characterization of the relationship between information gain and regret, but instead via a quantity called "information ratio". Furthermore, we consider planning with open-loop belief dynamics rather than Thompson sampling and our focus is on analyzing EFE.

Beyond information gathering heuristics, there is a family of POMDP extensions called active sensing or $\rho$-POMDP \citep{araya2010pomdp}, where the reward function is directly defined on beliefs. These POMDPs are typically used to model settings where the reduction in belief entropy is the primary goal, such as in the exploration of an area, and a goal-related reward can be optionally added. Without loss of generality, we can define this reward as the one-step ahead belief entropy:
\begin{align}
\begin{split}
    R^{AS}(b, a) &= -\mathbb{E}_{P(o'|b, a)}[\mathbb{H}[b'(s'|o', b, a)]] \\
    &\propto \mathbb{E}_{P(o'|b, a)}[\mathbb{E}_{b'(s'|o', b, a)}[\log b'(s'|o', b, a) - \log \frac{1}{|\mathcal{S}|}]] \\
    &= \mathbb{E}_{P(o'|b, a)}[\mathbb{KL}[b'(s'|o', b, a) || \tilde{b}(s')]] \,.
\end{split}
\end{align}
where $\tilde{b}(s') = \frac{1}{|\mathcal{S}|}$ is a uniform prior belief. This shows that the active sensing objective can be written as a special type of information gain that is evaluated against a uniform prior belief, thus resembling the EFE objective. An attractive property of this objective is that it is convex in the belief, and thus is the value function, which makes the agent potentially less distracted by information gain when task rewards are introduced. A further difference is that it uses closed-loop belief dynamics which enables better optimization of the information objective.

\subsection{Objective Specification in Active Inference}\label{sec:aif_obj_specification}
The objective functions in active inference have been subject to various interpretations since its inception in the late 2000's and have only slowed down relatively recently \citep{gottwald2020two}. The EFE objective, which first appeared in the literature as early as 2015 in \citep{friston2015active}, was initially motivated by an intuitive argument that "free energy minimizing agents should choose actions to minimize (expected) free energy". However, far from being heuristic, the EFE objective is rooted in the free energy principle which adopts a physics and information geometric perspective, rather than a decision theoretic perspective, on agent behavior \citep{friston2023path, friston2023free, barp2022geometric}, in which case open-loop belief dynamics is the natural outcome. It should be mentioned, however, that the information geometric derivation of EFE relies on a "precise agent" assumption on the environment in which future actions and observations are assumed to have matching entropy \citep{barp2022geometric, da2024active}. It remains open whether this assumption is satisfied in real environments. 

Recently, \citet{friston2021sophisticated} introduced a "sophisticated" version of EFE as an improved planning objective for active inference agents, where instead of evaluating EFE based on future state marginals, EFE is evaluated based on future posterior beliefs $Q(s_{t}|o_{t}, a_{0:t-1})$. This means that the belief MDP underlying the sophisticated EFE uses the closed-loop belief dynamics rather than the open-loop belief dynamics in the vanilla EFE, however, the information gain term is still used in the reward function. This means that we can no longer view sophisticated EFE as an approximation to the Bayes optimal policy. Rather, the combination of pragmatic value and closed-loop belief dynamics renders parts of sophisticated EFE exactly equal to the Bayes optimal belief MDP, until the equivalence is "broken" again by the additional information gain term. Does this mean the agent may be motivated to acquire too much information while compromising task performance? A simple manipulation shows that if we define the preference distribution as the exponentiated reward multiplied by a negative temperature parameter $\lambda$ $\tilde{P}(o) \propto \exp(\lambda \tilde{R}(o))$, then the EFE reward becomes proportional to a weighted combination of reward and information gain:
\begin{align}
\begin{split}
    \tilde{R}(s, a)
    &\propto \sum_{s'}P(s'|s, a)\sum_{o'}P(o'|s')\lambda \tilde{R}(o')\\
    &= \lambda \tilde{R}(s, a) \,,\\
    R^{EFE}(b, a) &\propto \sum_{s}b(s)\tilde{R}(s, a) + \frac{1}{\lambda}IG(b, a) \,,
\end{split}
\end{align}
where choosing a high $\lambda \rightarrow \infty$ corresponds to purely optimizing reward. However, this does mean that when $\lambda$ is not sufficiently high, in which case the objective highly resembles active sensing, the agent may be distracted. But whether this will be the case depends on the actual environment. Thus, similar to assumption \ref{assumption:reward_specification}, achieving Bayes optimal behavior requires setting the preference in such a way that the cumulative reward outweighs cumulative information.

Another perspective on the EFE objective is that the agent performs distribution matching as opposed to reward maximization \citep{da2023reward} where the agent additionally seeks out diverse states or observations. This can be seen from a rearrangement of the pragmatic-epistemic decomposition of the EFE objective:
\begin{align}
\begin{split}
    R^{EFE}(b, a) &= \mathbb{E}_{P(o'|b, a)}[\log\tilde{P}(o')] + \mathbb{E}_{P(o'|b, a)}[\mathbb{KL}[b'(s'|o', b, a) || b'(s'|b, a)]] \\
    &= -\underbrace{\mathbb{KL}[P(o'|b, a) || \tilde{P}(o')]}_{\text{Risk}} - \underbrace{\mathbb{E}_{P(s'|b, a)}[\mathbb{H}[P(o'|s')]]}_{\text{Ambiguity}} \,.
\end{split}
\end{align}
This is the well-known risk-ambiguity decomposition of EFE \citep{sajid2021active}, where the first term "risk" measures the KL divergence of the predicted observation distribution from the preferred observation distribution and the second term "ambiguity" measures the entropy of observations expected under predicted future states. 

In the MDP setting, with closed-loop belief updating, the objective reduces to the following due to no ambiguity, which is precisely the well-known distribution matching objective \citep{hafner2020action}:
\begin{align}
    R^{EFE}(s, a) &= -\mathbb{KL}[P(s'|s, a) || \tilde{P}(s')] \,,
\end{align}
This objective has been shown to enhance exploration and test-time adaptation in an RL setting \citep{lee2019efficient}.

Again, as shown in \citep{da2023reward}, distribution matching and reward maximization can be interpolated using a temperature parameter on the state preference $\tilde{P}(s) \propto \exp(\lambda \tilde{R}(s))$:
\begin{align}\label{eq:efe_mdp_obj}
\begin{split}
    R^{EFE}(s, a)
    &= \mathbb{E}_{P(s'|s, a)}[\log \tilde{P}(s')] + \mathbb{H}[P(s'|s, a)] \\
    &\propto \mathbb{E}_{P(s'|s, a)}[\tilde{R}(s')] + \frac{1}{\lambda}\mathbb{H}[P(s'|s, a)] \,.
\end{split}
\end{align}
In this setting, the temperature parameter $\lambda$ represents the allowed dispersion around the optimal behavior (or path of least action) specified by the first expected reward term in (\ref{eq:efe_mdp_obj}). Alternatively, it can be interpreted as the tightness of the (soft) constraint to abide by optimal behavior, following the constrained maximum entropy view of the free energy principle \citep{friston2023path}. 

Putting together these perspectives, it appears that the notion of "Bayes optimal" in the spirit of active inference (in closed-loop), as well as extensions of POMDPs, may not be restricted to the usual sense of Bayesian decision theory (i.e., maximizing utility; \citealt{howard1966information, raiffa2000applied, berger2013statistical}); it may also apply to that of Bayesian optimal design (i.e., maximizing information gain; \citealt{lindley1956measure, mackay1992information}) and principle of maximum caliber (i.e., maximizing coverage; \citealt{jaynes1980minimum}). 

\section{Conclusion}
In this paper, we study the theoretical connection between active inference and reinforcement learning and show that the epistemic value in the EFE objective of active inference can be seen as an approximation to the Bayes optimal RL policy in POMDPs, achieving a linear improvement in regret compared to a naive policy which doesn't take into account the value of information. 
The results also suggest that, from the perspective of RL, the specification of EFE needs to balance reward with information gain in the environment, via an appropriate temperature parameter ($\lambda$). Conversely, from the perspective of active inference, an EFE minimizing agent will pursue a Bayes optimal RL policy, under a suitable temperature parameter. 
This conclusion might have been anticipated by one reading of the complete class theorem \citep{wald1947essentially, brown1981complete}; namely, for any pair of reward function and choices, there exists some prior beliefs that render the choices Bayes optimal, in a decision theoretic sense \citep{berger2013statistical}. 

\section*{Acknowledgement}
The author would like to thank Alex Kiefer, Axel Constant, David Hyland, Karl Friston, Lance Da Costa, Peter Waade, Ryan Singh, Sanjeev Namjoshi, and Shohei Wakayama for helpful feedback.

\bibliography{ref.bib}

\begin{thebibliography}{62}
\providecommand{\natexlab}[1]{#1}
\providecommand{\url}[1]{\texttt{#1}}
\expandafter\ifx\csname urlstyle\endcsname\relax
  \providecommand{\doi}[1]{doi: #1}\else
  \providecommand{\doi}{doi: \begingroup \urlstyle{rm}\Url}\fi

\bibitem[Agarwal et~al.(2019)Agarwal, Jiang, Kakade, and Sun]{agarwal2019reinforcement}
A.~Agarwal, N.~Jiang, S.~M. Kakade, and W.~Sun.
\newblock Reinforcement learning: Theory and algorithms.
\newblock \emph{CS Dept., UW Seattle, Seattle, WA, USA, Tech. Rep}, 32:\penalty0 96, 2019.

\bibitem[Araya et~al.(2010)Araya, Buffet, Thomas, and Charpillet]{araya2010pomdp}
M.~Araya, O.~Buffet, V.~Thomas, and F.~Charpillet.
\newblock A pomdp extension with belief-dependent rewards.
\newblock \emph{Advances in neural information processing systems}, 23, 2010.

\bibitem[Barp et~al.(2022)Barp, Da~Costa, Fran{\c{c}}a, Friston, Girolami, Jordan, and Pavliotis]{barp2022geometric}
A.~Barp, L.~Da~Costa, G.~Fran{\c{c}}a, K.~Friston, M.~Girolami, M.~I. Jordan, and G.~A. Pavliotis.
\newblock Geometric methods for sampling, optimization, inference, and adaptive agents.
\newblock In \emph{Handbook of Statistics}, volume~46, pages 21--78. Elsevier, 2022.

\bibitem[Belghazi et~al.(2018)Belghazi, Baratin, Rajeshwar, Ozair, Bengio, Courville, and Hjelm]{belghazi2018mutual}
M.~I. Belghazi, A.~Baratin, S.~Rajeshwar, S.~Ozair, Y.~Bengio, A.~Courville, and D.~Hjelm.
\newblock Mutual information neural estimation.
\newblock In \emph{International conference on machine learning}, pages 531--540. PMLR, 2018.

\bibitem[Berger(2013)]{berger2013statistical}
J.~O. Berger.
\newblock \emph{Statistical decision theory and Bayesian analysis}.
\newblock Springer Science \& Business Media, 2013.

\bibitem[Blei et~al.(2017)Blei, Kucukelbir, and McAuliffe]{blei2017variational}
D.~M. Blei, A.~Kucukelbir, and J.~D. McAuliffe.
\newblock Variational inference: A review for statisticians.
\newblock \emph{Journal of the American statistical Association}, 112\penalty0 (518):\penalty0 859--877, 2017.

\bibitem[Brown(1981)]{brown1981complete}
L.~D. Brown.
\newblock A complete class theorem for statistical problems with finite sample spaces.
\newblock \emph{The Annals of Statistics}, pages 1289--1300, 1981.

\bibitem[Cassandra et~al.(1996)Cassandra, Kaelbling, and Kurien]{cassandra1996acting}
A.~R. Cassandra, L.~P. Kaelbling, and J.~A. Kurien.
\newblock Acting under uncertainty: Discrete bayesian models for mobile-robot navigation.
\newblock In \emph{Proceedings of IEEE/RSJ International Conference on Intelligent Robots and Systems. IROS'96}, volume~2, pages 963--972. IEEE, 1996.

\bibitem[Chakraborty et~al.(2023)Chakraborty, Bedi, Koppel, Wang, Huang, and Manocha]{chakraborty2023steering}
S.~Chakraborty, A.~S. Bedi, A.~Koppel, M.~Wang, F.~Huang, and D.~Manocha.
\newblock Steering: Stein information directed exploration for model-based reinforcement learning.
\newblock \emph{arXiv preprint arXiv:2301.12038}, 2023.

\bibitem[Champion et~al.(2024)Champion, Bowman, Markovi{\'c}, and Grze{\'s}]{champion2024reframing}
T.~Champion, H.~Bowman, D.~Markovi{\'c}, and M.~Grze{\'s}.
\newblock Reframing the expected free energy: Four formulations and a unification.
\newblock \emph{arXiv preprint arXiv:2402.14460}, 2024.

\bibitem[Da~Costa et~al.(2020)Da~Costa, Parr, Sajid, Veselic, Neacsu, and Friston]{da2020active}
L.~Da~Costa, T.~Parr, N.~Sajid, S.~Veselic, V.~Neacsu, and K.~Friston.
\newblock Active inference on discrete state-spaces: A synthesis.
\newblock \emph{Journal of Mathematical Psychology}, 99:\penalty0 102447, 2020.

\bibitem[Da~Costa et~al.(2023)Da~Costa, Sajid, Parr, Friston, and Smith]{da2023reward}
L.~Da~Costa, N.~Sajid, T.~Parr, K.~Friston, and R.~Smith.
\newblock Reward maximization through discrete active inference.
\newblock \emph{Neural Computation}, 35\penalty0 (5):\penalty0 807--852, 2023.

\bibitem[Da~Costa et~al.(2024)Da~Costa, Tenka, Zhao, and Sajid]{da2024active}
L.~Da~Costa, S.~Tenka, D.~Zhao, and N.~Sajid.
\newblock Active inference as a model of agency.
\newblock \emph{arXiv preprint arXiv:2401.12917}, 2024.

\bibitem[Doshi-Velez and Konidaris(2016)]{doshi2016hidden}
F.~Doshi-Velez and G.~Konidaris.
\newblock Hidden parameter markov decision processes: A semiparametric regression approach for discovering latent task parametrizations.
\newblock In \emph{IJCAI: proceedings of the conference}, volume 2016, page 1432. NIH Public Access, 2016.

\bibitem[Duan et~al.(2016)Duan, Schulman, Chen, Bartlett, Sutskever, and Abbeel]{duan2016rl}
Y.~Duan, J.~Schulman, X.~Chen, P.~L. Bartlett, I.~Sutskever, and P.~Abbeel.
\newblock $rl^2$: Fast reinforcement learning via slow reinforcement learning.
\newblock \emph{arXiv preprint arXiv:1611.02779}, 2016.

\bibitem[Duff(2002)]{duff2002optimal}
M.~O. Duff.
\newblock \emph{Optimal Learning: Computational procedures for Bayes-adaptive Markov decision processes}.
\newblock University of Massachusetts Amherst, 2002.

\bibitem[Engstr{\"o}m et~al.(2024)Engstr{\"o}m, Wei, McDonald, Garcia, O'Kelly, and Johnson]{engstrom2024resolving}
J.~Engstr{\"o}m, R.~Wei, A.~D. McDonald, A.~Garcia, M.~O'Kelly, and L.~Johnson.
\newblock Resolving uncertainty on the fly: modeling adaptive driving behavior as active inference.
\newblock \emph{Frontiers in neurorobotics}, 18:\penalty0 1341750, 2024.

\bibitem[Flaspohler et~al.(2020)Flaspohler, Roy, and Fisher~III]{flaspohler2020belief}
G.~Flaspohler, N.~A. Roy, and J.~W. Fisher~III.
\newblock Belief-dependent macro-action discovery in pomdps using the value of information.
\newblock \emph{Advances in Neural Information Processing Systems}, 33:\penalty0 11108--11118, 2020.

\bibitem[Friston(2010)]{friston2010free}
K.~Friston.
\newblock The free-energy principle: a unified brain theory?
\newblock \emph{Nature reviews neuroscience}, 11\penalty0 (2):\penalty0 127--138, 2010.

\bibitem[Friston et~al.(2015)Friston, Rigoli, Ognibene, Mathys, Fitzgerald, and Pezzulo]{friston2015active}
K.~Friston, F.~Rigoli, D.~Ognibene, C.~Mathys, T.~Fitzgerald, and G.~Pezzulo.
\newblock Active inference and epistemic value.
\newblock \emph{Cognitive neuroscience}, 6\penalty0 (4):\penalty0 187--214, 2015.

\bibitem[Friston et~al.(2017)Friston, FitzGerald, Rigoli, Schwartenbeck, and Pezzulo]{friston2017active}
K.~Friston, T.~FitzGerald, F.~Rigoli, P.~Schwartenbeck, and G.~Pezzulo.
\newblock Active inference: a process theory.
\newblock \emph{Neural computation}, 29\penalty0 (1):\penalty0 1--49, 2017.

\bibitem[Friston et~al.(2021)Friston, Da~Costa, Hafner, Hesp, and Parr]{friston2021sophisticated}
K.~Friston, L.~Da~Costa, D.~Hafner, C.~Hesp, and T.~Parr.
\newblock Sophisticated inference.
\newblock \emph{Neural Computation}, 33\penalty0 (3):\penalty0 713--763, 2021.

\bibitem[Friston et~al.(2023{\natexlab{a}})Friston, Da~Costa, Sajid, Heins, Ueltzh{\"o}ffer, Pavliotis, and Parr]{friston2023free}
K.~Friston, L.~Da~Costa, N.~Sajid, C.~Heins, K.~Ueltzh{\"o}ffer, G.~A. Pavliotis, and T.~Parr.
\newblock The free energy principle made simpler but not too simple.
\newblock \emph{Physics Reports}, 1024:\penalty0 1--29, 2023{\natexlab{a}}.

\bibitem[Friston et~al.(2023{\natexlab{b}})Friston, Da~Costa, Sakthivadivel, Heins, Pavliotis, Ramstead, and Parr]{friston2023path}
K.~Friston, L.~Da~Costa, D.~A. Sakthivadivel, C.~Heins, G.~A. Pavliotis, M.~Ramstead, and T.~Parr.
\newblock Path integrals, particular kinds, and strange things.
\newblock \emph{Physics of Life Reviews}, 2023{\natexlab{b}}.

\bibitem[Gottwald and Braun(2020)]{gottwald2020two}
S.~Gottwald and D.~A. Braun.
\newblock The two kinds of free energy and the bayesian revolution.
\newblock \emph{PLoS computational biology}, 16\penalty0 (12):\penalty0 e1008420, 2020.

\bibitem[Hafner et~al.(2020)Hafner, Ortega, Ba, Parr, Friston, and Heess]{hafner2020action}
D.~Hafner, P.~A. Ortega, J.~Ba, T.~Parr, K.~Friston, and N.~Heess.
\newblock Action and perception as divergence minimization.
\newblock \emph{arXiv preprint arXiv:2009.01791}, 2020.

\bibitem[Hao and Lattimore(2022)]{hao2022regret}
B.~Hao and T.~Lattimore.
\newblock Regret bounds for information-directed reinforcement learning.
\newblock \emph{Advances in neural information processing systems}, 35:\penalty0 28575--28587, 2022.

\bibitem[Hauskrecht(2000)]{hauskrecht2000value}
M.~Hauskrecht.
\newblock Value-function approximations for partially observable markov decision processes.
\newblock \emph{Journal of artificial intelligence research}, 13:\penalty0 33--94, 2000.

\bibitem[Howard(1966)]{howard1966information}
R.~A. Howard.
\newblock Information value theory.
\newblock \emph{IEEE Transactions on systems science and cybernetics}, 2\penalty0 (1):\penalty0 22--26, 1966.

\bibitem[Jaynes(1980)]{jaynes1980minimum}
E.~T. Jaynes.
\newblock The minimum entropy production principle.
\newblock \emph{Annual Review of Physical Chemistry}, 31\penalty0 (1):\penalty0 579--601, 1980.

\bibitem[Kaelbling et~al.(1998)Kaelbling, Littman, and Cassandra]{kaelbling1998planning}
L.~P. Kaelbling, M.~L. Littman, and A.~R. Cassandra.
\newblock Planning and acting in partially observable stochastic domains.
\newblock \emph{Artificial intelligence}, 101\penalty0 (1-2):\penalty0 99--134, 1998.

\bibitem[Kakade and Langford(2002)]{kakade2002approximately}
S.~Kakade and J.~Langford.
\newblock Approximately optimal approximate reinforcement learning.
\newblock In \emph{Proceedings of the Nineteenth International Conference on Machine Learning}, pages 267--274, 2002.

\bibitem[Kearns and Singh(2002)]{kearns2002near}
M.~Kearns and S.~Singh.
\newblock Near-optimal reinforcement learning in polynomial time.
\newblock \emph{Machine learning}, 49:\penalty0 209--232, 2002.

\bibitem[Koudahl et~al.(2021)Koudahl, Kouw, and de~Vries]{koudahl2021epistemics}
M.~T. Koudahl, W.~M. Kouw, and B.~de~Vries.
\newblock On epistemics in expected free energy for linear gaussian state space models.
\newblock \emph{Entropy}, 23\penalty0 (12):\penalty0 1565, 2021.

\bibitem[Lanillos et~al.(2021)Lanillos, Meo, Pezzato, Meera, Baioumy, Ohata, Tschantz, Millidge, Wisse, Buckley, et~al.]{lanillos2021active}
P.~Lanillos, C.~Meo, C.~Pezzato, A.~A. Meera, M.~Baioumy, W.~Ohata, A.~Tschantz, B.~Millidge, M.~Wisse, C.~L. Buckley, et~al.
\newblock Active inference in robotics and artificial agents: Survey and challenges.
\newblock \emph{arXiv preprint arXiv:2112.01871}, 2021.

\bibitem[Lee et~al.(2019)Lee, Eysenbach, Parisotto, Xing, Levine, and Salakhutdinov]{lee2019efficient}
L.~Lee, B.~Eysenbach, E.~Parisotto, E.~Xing, S.~Levine, and R.~Salakhutdinov.
\newblock Efficient exploration via state marginal matching.
\newblock \emph{arXiv preprint arXiv:1906.05274}, 2019.

\bibitem[Lindley(1956)]{lindley1956measure}
D.~V. Lindley.
\newblock On a measure of the information provided by an experiment.
\newblock \emph{The Annals of Mathematical Statistics}, 27\penalty0 (4):\penalty0 986--1005, 1956.

\bibitem[Littman et~al.(1995)Littman, Cassandra, and Kaelbling]{littman1995learning}
M.~L. Littman, A.~R. Cassandra, and L.~P. Kaelbling.
\newblock Learning policies for partially observable environments: Scaling up.
\newblock In \emph{Machine Learning Proceedings 1995}, pages 362--370. Elsevier, 1995.

\bibitem[Liu et~al.(2022)Liu, Chung, Szepesv{\'a}ri, and Jin]{liu2022partially}
Q.~Liu, A.~Chung, C.~Szepesv{\'a}ri, and C.~Jin.
\newblock When is partially observable reinforcement learning not scary?
\newblock In \emph{Conference on Learning Theory}, pages 5175--5220. PMLR, 2022.

\bibitem[Lu et~al.(2023)Lu, Van~Roy, Dwaracherla, Ibrahimi, Osband, Wen, et~al.]{lu2023reinforcement}
X.~Lu, B.~Van~Roy, V.~Dwaracherla, M.~Ibrahimi, I.~Osband, Z.~Wen, et~al.
\newblock Reinforcement learning, bit by bit.
\newblock \emph{Foundations and Trends{\textregistered} in Machine Learning}, 16\penalty0 (6):\penalty0 733--865, 2023.

\bibitem[MacKay(1992)]{mackay1992information}
D.~J. MacKay.
\newblock Information-based objective functions for active data selection.
\newblock \emph{Neural computation}, 4\penalty0 (4):\penalty0 590--604, 1992.

\bibitem[Malekzadeh and Plataniotis(2022)]{malekzadeh2022active}
P.~Malekzadeh and K.~N. Plataniotis.
\newblock Active inference and reinforcement learning: A unified infer-ence on continuous state and action spaces under partially observ-ability.
\newblock \emph{arXiv preprint arXiv:2212.07946}, 2022.

\bibitem[Mazzaglia et~al.(2022)Mazzaglia, Verbelen, Catal, and Dhoedt]{mazzaglia2022free}
P.~Mazzaglia, T.~Verbelen, O.~Catal, and B.~Dhoedt.
\newblock The free energy principle for perception and action: A deep learning perspective.
\newblock \emph{Entropy}, 24\penalty0 (2):\penalty0 301, 2022.

\bibitem[Millidge(2020)]{millidge2020deep}
B.~Millidge.
\newblock Deep active inference as variational policy gradients.
\newblock \emph{Journal of Mathematical Psychology}, 96:\penalty0 102348, 2020.

\bibitem[Millidge et~al.(2020)Millidge, Tschantz, Seth, and Buckley]{millidge2020relationship}
B.~Millidge, A.~Tschantz, A.~K. Seth, and C.~L. Buckley.
\newblock On the relationship between active inference and control as inference.
\newblock In \emph{Active Inference: First International Workshop, IWAI 2020, Co-located with ECML/PKDD 2020, Ghent, Belgium, September 14, 2020, Proceedings 1}, pages 3--11. Springer, 2020.

\bibitem[Parr et~al.(2022)Parr, Pezzulo, and Friston]{parr2022active}
T.~Parr, G.~Pezzulo, and K.~J. Friston.
\newblock \emph{Active inference: the free energy principle in mind, brain, and behavior}.
\newblock MIT Press, 2022.

\bibitem[Raiffa and Schlaifer(2000)]{raiffa2000applied}
H.~Raiffa and R.~Schlaifer.
\newblock \emph{Applied statistical decision theory}, volume~78.
\newblock John Wiley \& Sons, 2000.

\bibitem[Ross et~al.(2011)Ross, Gordon, and Bagnell]{ross2011reduction}
S.~Ross, G.~Gordon, and D.~Bagnell.
\newblock A reduction of imitation learning and structured prediction to no-regret online learning.
\newblock In \emph{Proceedings of the fourteenth international conference on artificial intelligence and statistics}, pages 627--635. JMLR Workshop and Conference Proceedings, 2011.

\bibitem[Roy et~al.(2005)Roy, Gordon, and Thrun]{roy2005finding}
N.~Roy, G.~Gordon, and S.~Thrun.
\newblock Finding approximate pomdp solutions through belief compression.
\newblock \emph{Journal of artificial intelligence research}, 23:\penalty0 1--40, 2005.

\bibitem[Russo and Van~Roy(2018)]{russo2018learning}
D.~Russo and B.~Van~Roy.
\newblock Learning to optimize via information-directed sampling.
\newblock \emph{Operations Research}, 66\penalty0 (1):\penalty0 230--252, 2018.

\bibitem[Sajid et~al.(2021)Sajid, Ball, Parr, and Friston]{sajid2021active}
N.~Sajid, P.~J. Ball, T.~Parr, and K.~J. Friston.
\newblock Active inference: demystified and compared.
\newblock \emph{Neural computation}, 33\penalty0 (3):\penalty0 674--712, 2021.

\bibitem[Schw{\"o}bel et~al.(2018)Schw{\"o}bel, Kiebel, and Markovi{\'c}]{schwobel2018active}
S.~Schw{\"o}bel, S.~Kiebel, and D.~Markovi{\'c}.
\newblock Active inference, belief propagation, and the bethe approximation.
\newblock \emph{Neural computation}, 30\penalty0 (9):\penalty0 2530--2567, 2018.

\bibitem[Smith et~al.(2021)Smith, Badcock, and Friston]{smith2021recent}
R.~Smith, P.~Badcock, and K.~J. Friston.
\newblock Recent advances in the application of predictive coding and active inference models within clinical neuroscience.
\newblock \emph{Psychiatry and Clinical Neurosciences}, 75\penalty0 (1):\penalty0 3--13, 2021.

\bibitem[Sutton and Barto(2018)]{sutton2018reinforcement}
R.~S. Sutton and A.~G. Barto.
\newblock \emph{Reinforcement learning: An introduction}.
\newblock MIT press, 2018.

\bibitem[Tomczak and Welling(2018)]{tomczak2018vae}
J.~Tomczak and M.~Welling.
\newblock Vae with a vampprior.
\newblock In \emph{International conference on artificial intelligence and statistics}, pages 1214--1223. PMLR, 2018.

\bibitem[Tschantz et~al.(2020)Tschantz, Baltieri, Seth, and Buckley]{tschantz2020scaling}
A.~Tschantz, M.~Baltieri, A.~K. Seth, and C.~L. Buckley.
\newblock Scaling active inference.
\newblock In \emph{2020 international joint conference on neural networks (ijcnn)}, pages 1--8. IEEE, 2020.

\bibitem[Vemula et~al.(2023)Vemula, Song, Singh, Bagnell, and Choudhury]{vemula2023virtues}
A.~Vemula, Y.~Song, A.~Singh, D.~Bagnell, and S.~Choudhury.
\newblock The virtues of laziness in model-based rl: A unified objective and algorithms.
\newblock In \emph{International Conference on Machine Learning}, pages 34978--35005. PMLR, 2023.

\bibitem[Wald(1947)]{wald1947essentially}
A.~Wald.
\newblock An essentially complete class of admissible decision functions.
\newblock \emph{The Annals of Mathematical Statistics}, pages 549--555, 1947.

\bibitem[Watson et~al.(2020)Watson, Imohiosen, and Peters]{watson2020active}
J.~Watson, A.~Imohiosen, and J.~Peters.
\newblock Active inference or control as inference? a unifying view.
\newblock \emph{arXiv preprint arXiv:2010.00262}, 2020.

\bibitem[Wei et~al.(2023)Wei, Zeng, Li, Garcia, McDonald, and Hong]{wei2023bayesian}
R.~Wei, S.~Zeng, C.~Li, A.~Garcia, A.~D. McDonald, and M.~Hong.
\newblock A bayesian approach to robust inverse reinforcement learning.
\newblock In \emph{Conference on Robot Learning}, pages 2304--2322. PMLR, 2023.

\bibitem[Winn et~al.(2005)Winn, Bishop, and Jaakkola]{winn2005variational}
J.~Winn, C.~M. Bishop, and T.~Jaakkola.
\newblock Variational message passing.
\newblock \emph{Journal of Machine Learning Research}, 6\penalty0 (4), 2005.

\bibitem[Zintgraf et~al.(2019)Zintgraf, Shiarlis, Igl, Schulze, Gal, Hofmann, and Whiteson]{zintgraf2019varibad}
L.~Zintgraf, K.~Shiarlis, M.~Igl, S.~Schulze, Y.~Gal, K.~Hofmann, and S.~Whiteson.
\newblock Varibad: A very good method for bayes-adaptive deep rl via meta-learning.
\newblock \emph{arXiv preprint arXiv:1910.08348}, 2019.

\end{thebibliography}

\newpage
\appendix\label{sec:appx}

\section{Appendix}
\subsection{Proofs for Section \ref{sec:act_inf}}
\paragraph{Derivation of $Q(s_{1:T}|a_{0:T-1})$ in (\ref{eq:efe_q_def}) from variational inference} We aim to obtain a predictive distribution over future states $s_{1:T}$ given an action sequence $a_{0:T-1}$ using variational inference. Typically, active inference assumes a mean-field factorization of the variational distribution $Q(s_{1:T}|a_{0:T-1}) = \prod_{t=1}^{T}Q(s_{t}|a_{0:T-1})$. Since there is no observation and thus no likelihood term, the variational free energy $\mathcal{F}$ can be written as:
\begin{align}
\begin{split}
    \mathcal{F}(Q) &= \mathbb{E}_{Q(s_{1:T}|a_{0:T-1})}[\log Q(s_{1:T}|a_{0:T-1}) - \log P(s_{1:T}|a_{0:T-1})] \\
    &= \mathbb{E}_{Q(s_{1:T}|a_{0:T-1})}\left[\sum_{t=1}^{T}\left(\log Q(s_{t}|a_{0:T-1}) - \log P(s_{t}|s_{t-1}, a_{t-1})\right)\right] \\
    &= \sum_{t=1}^{T}\mathbb{E}_{Q(s_{t-1:t}|a_{0:T-1})}[\log Q(s_{t}|a_{0:T-1}) - \log P(s_{t}|s_{t-1}, a_{t-1})] \,.
\end{split}
\end{align}

From \citep{winn2005variational}, we know the optimal variational distribution has the form:
\begin{align}
\begin{split}
    Q(s_{t}|a_{0:T-1}) &\propto \exp(\mathbb{E}_{Q(s_{t-1}|a_{0:T-1})}[\log P(s_{t}|s_{t-1}, a_{t-1})]) \\
    &\approx \exp(\log \mathbb{E}_{Q(s_{t-1}|a_{0:T-1})}[P(s_{t}| s_{t-1}, a_{t-1})]) \\
    &= \sum_{s_{t-1}}P(s_{t}| s_{t-1}, a_{t-1})Q(s_{t-1}|a_{0:T-1}) \\
    &:= Q(s_{t}|Q_{t-1}, a_{t-1}) \,.
\end{split}
\end{align}
which recovers the definition in (\ref{eq:efe_q_def}). The approximation in the second line is due to Jensen's inequality and does not significantly affect our results, because we know from the variational inference literature that the optimal variational distribution must be equal to that of exact inference, which is given by the last line. This also matches the implementation in Pymdp\footnote{\url{https://github.com/infer-actively/pymdp}}, which is one of the main software repositories for active inference.

\paragraph{Active inference and QMDP} It is crucial to have a precise definition of the distributions $Q(s_{0:T}|a_{0:T-1})$ and $Q(o_{0:T}, s_{0:T}|a_{0:T-1})$. In the main text, we have specified these as the product of marginal distributions over states and observations. Here, we briefly study the consequences of defining these as the joint distributions:
\begin{align}
\begin{split}
    Q(s_{0:T}|a_{0:T-1}) &= b(s_{0})\prod_{t=1}^{T}P(s_{t}|s_{t-1}, a_{t-1}) \,, \\
    Q(o_{0:T}, s_{0:T}|a_{0:T-1}) &= b(s_{0})P(o_{0}|s_{0})\prod_{t=1}^{T}P(s_{t}|s_{t-1}, a_{t-1})P(o_{t}|s_{t}) \,. \\
\end{split}
\end{align}

We start by factorizing the full EFE objective in (\ref{eq:efe_full}) as:
\begin{align}
\begin{split}
    &EFE(a_{0:T-1}) \\
    &= \mathbb{E}_{Q(o_{1:T}, s_{1:T}|a_{0:T-1})}[\log Q(s_{1:T}|a_{0:T-1}) - \log \Tilde{P}(o_{1:T}, s_{1:T})] \\
    &= \mathbb{E}_{Q(o_{1:T}, s_{1:T}|a_{0:T-1})}\left[\sum_{t=1}^{T}\left(\log P(s_{t}|s_{t-1}, a_{t-1}) - \log \Tilde{P}(o_{t}, s_{t})\right)\right] \\
    &= \mathbb{E}_{b(s_{0})P(s_{1}|s_{0}, a_{0})P(o_{1}|s_{1})}\Bigg[\log P(s_{1}|s_{0}, a_{0}) - \log \tilde{P}(o_{1}, s_{1}) \\
    &\quad + \mathbb{E}_{Q(o_{2:T}, s_{2:T}|s_{0:1}, a_{1:T-1})}\left[\sum_{t=2}^{T}\left(\log P(s_{t}|s_{t-1}, a_{t-1}) - \log \Tilde{P}(o_{t}, s_{t})\right)\right]\Bigg] \\
    &= \mathbb{E}_{b(s_{0})P(s_{1}|s_{0}, a_{0})P(o_{1}|s_{1})}\left[\log P(s_{1}|s_{0}, a_{0}) - \log \tilde{P}(o_{1}, s_{1}) + EFE(a_{1:T-1})\right] \\
    &= \mathbb{E}_{b(s_0)}\left[\mathbb{E}_{P(s_{1}|s_{0}, a_{0})P(o_{1}|s_{1})}[\log P(s_{1}|s_{0}, a_{0}) - \log \tilde{P}(o_{1}, s_{1})] + \mathbb{E}_{P(s_{1}|s_{0}, a_{0})}[EFE(a_{1:T-1})]\right] \,.
\end{split}
\end{align}

This allows us to write down a recursive equation:
\begin{align}
\begin{split}
    Q(s_{t}, a_{t}) &= \underbrace{\mathbb{E}_{P(s_{t+1}|s_{t}, a_{t})P(o_{t+1}|s_{t+1})}[\log P(s_{t+1}|s_{t}, a_{t}) - \log \tilde{P}(o_{t+1}, s_{t+1})]}_{R(s_{t}, a_{t})} + \mathbb{E}_{P(s_{t+1}|s_{t}, a_{t})}[V(s_{t+1})] \,, \\
    V(s_{t}) &= \max_{a}Q(s_{t}, a_{t}) \,,
\end{split}
\end{align}
and
\begin{align}
    EFE(a_{0:T-1}) = \mathbb{E}_{b(s_{0})}[Q(s_{0}, a_{0})] \,.
\end{align}

This corresponds to what's known as the QMDP approximation in the POMDP literature \citep{littman1995learning}, which is known to overestimate the value of a belief by planning under the implicit assumption that future states are observed \citep{hauskrecht2000value}.

\paragraph{EFE bound and choice of preference} Despite being the most popular choice of EFE, the pragmatic-epistemic value decomposition (\ref{eq:efe_prag_epis}) is actually a bound on the full EFE defined in (\ref{eq:efe_full}). To show this, let's consider a single time step since both formulations can be decomposed across time steps. Recall that the pragmatic-epistemic decomposition assumes the following factorization of $\tilde{P}(o, s) = \tilde{P}(o)\tilde{P}(s|o)$. The full EFE can be written as:
\begin{align}
\begin{split}
    &EFE_{t}(a_{0:T-1}) = \mathbb{E}_{Q(o_{t}, s_{t}|a_{0:T-1})}[\log Q(s_{t}|a_{0:T-1}) - \log \tilde{P}(o_{t}, s_{t})]\\
    &= -\mathbb{E}_{Q(o_{t}|a_{0:T-1})}[\log \tilde{P}(o_{t})] - \mathbb{E}_{Q(o_{t}, s_{t}|a_{0:T-1})}[\log \tilde{P}(s_{t}|o_{t})] + \mathbb{E}_{Q(s_{t}|a_{0:T-1})}[Q(s_{t}|a_{0:T-1})] \\
    &= -\mathbb{E}_{Q(o_{t}|a_{0:T-1})}[\log \tilde{P}(o_{t})] + \mathbb{E}_{Q(o_{t}, s_{t}|a_{0:T-1})}[Q(s_{t}|o_{t}, a_{0:T-1})] - \mathbb{E}_{Q(o_{t}, s_{t}|a_{0:T-1})}[\log \tilde{P}(s_{t}|o_{t})] \\
    &\quad + \mathbb{E}_{Q(s_{t}|a_{0:T-1})}[Q(s_{t}|a_{0:T-1})] - \mathbb{E}_{Q(o_{t}, s_{t}|a_{0:T-1})}[Q(s_{t}|o_{t}, a_{0:T-1})] \\
    &= -\mathbb{E}_{Q(o_{t}|a_{0:T-1})}[\log \tilde{P}(o_{t})] + \mathbb{E}_{Q(o_{t}|a_{0:T-1})}\mathbb{KL}[Q(s_{t}|o_{t}, a_{0:T-1}) || \tilde{P}(s_{t}|o_{t})] \\
    &\quad - \mathbb{E}_{Q(o_{t}|a_{0:T-1})}[\mathbb{KL}[Q(s_{t}|o_{t}, a_{0:T-1}) || Q(s_{t}|a_{0:T-1})]] \\
    &\geq -\mathbb{E}_{Q(o_{t}|a_{0:T-1})}[\log \tilde{P}(o_{t})] - \mathbb{E}_{Q(o_{t}|a_{0:T-1})}[\mathbb{KL}[Q(s_{t}|o_{t}, a_{0:T-1}) || Q(s_{t}|a_{0:T-1})]] \,.
\end{split}
\end{align}

Thus, to keep the bound tight, we could set $\tilde{P}(s|o)$ as:
\begin{align}
\begin{split}
    \tilde{P}^{*}(s|o) &= \arg\min_{\tilde{P}(s|o)} \mathbb{E}_{Q(o_{t}|a_{0:T-1})}\mathbb{KL}[Q(s_{t}|o_{t}, a_{0:T-1}) || \tilde{P}(s_{t}|o_{t})] \\
    &\approx \arg\min_{\tilde{P}(s|o)} \mathbb{E}_{Q(o_{t}|a_{0:T-1})}\mathbb{KL}[\tilde{P}(s_{t}|o_{t}) || Q(s_{t}|o_{t}, a_{0:T-1})] \\
    &\propto \exp\left(\mathbb{E}_{Q(o_{t}|a_{0:T-1})}[\log Q(s_{t}|o_{t}, a_{0:T-1})]\right) \,,
\end{split}
\end{align}
where the approximation in the second line assumes the forward and reverse KL divergences have similar solutions. The result on the last line is sometimes referred to as the aggregate posterior \citep{tomczak2018vae}. However, since the aggregate posterior depends on the action sequence evaluated, the tightest bound is achieved by an aggregate posterior that updates during each EFE optimization step to ensure that the final aggregate posterior is evaluated under the \emph{optimal} action sequence.

\begin{proposition}
(Active inference policy; restate of proposition \ref{prop:act_inf_policy}) The EFE achieved by the optimal action sequence can be equivalently achieved by a time-indexed belief-action policy $\pi(a_{t}|Q_{t})$.
\end{proposition}
\begin{proof}
While the proof in the main text is given by characterizing the EFE objective as a belief MDP, we give an alternative proof here based on Bellman optimality for the full EFE objective in (\ref{eq:efe_full}) starting with the base case:
\begin{align}
    EFE(a_{T-1}, Q_{T-1}) = \mathbb{E}_{Q(o_{T}, s_{T}|Q_{T-1}, a_{T-1})}[\log Q(s_{T}|Q_{T-1}, a_{T-1}) - \log \tilde{P}(o_{T}, s_{T})] \,.
\end{align}
It is easy to see that
\begin{align}
    \min_{a_{T-1}}EFE(a_{T-1}, Q_{T-1}) = \max_{\pi_{T-1}}\sum_{a_{T-1}}\pi(a_{T-1}|Q_{T-1})EFE(a_{T-1}, Q_{T-1}) \,,
\end{align}
where the optimal policy is $\pi^{*}_{T-1}(a_{T-1}|Q_{T-1}) = \delta(a_{T-1} - \arg\min_{\tilde{a}_{T-1}}EFE(\tilde{a}_{T-1}, Q_{T-1}))$.

Applying the identity recursively, we have:
\begin{align}
\begin{split}
    &\min_{\pi_{t}}\mathbb{E}_{\pi(a_{t}|Q_{t})}[EFE(a_{t}, Q_{t})] = \min_{\pi_{t}}\mathbb{E}_{\pi(a_{t}|Q_{t})}\bigg\{\\
    &\mathbb{E}_{Q(o_{t+1}, s_{t+1}|Q_{t}, a_{t})}[\log Q(s_{t+1}|Q_{t}, a_{t}) - \log \tilde{P}(o_{t+1}, s_{t+1})] + \mathbb{E}_{\pi^{*}(a_{t}|Q_{t})}[EFE(a_{t+1}, Q_{t+1})]\bigg\} \,.
\end{split}
\end{align}
The optimal policy at each step can be obtained by $\pi(a_{t}|Q_{t}) = \delta(a_{t} - \arg\min_{\tilde{a}_{t}}EFE(\tilde{a}_{t}, Q_{t}))$. 
\end{proof}

\begin{proposition}
(Restate of proposition \ref{eq:prop_efe_reward_concavity}) The EFE reward function as defined in (\ref{eq:efe_belief_mdp_rwd}) is concave in the belief.
\end{proposition}

\begin{proof}
Recall the EFE reward is defined as:
\begin{align}
    R(b, a) = \mathbb{E}_{P(o'|b, a)}[\log\tilde{P}(o')] + \mathbb{E}_{P(o'|b, a)}[\mathbb{KL}[b'(s'|o', b, a) || b'(s'|b, a)]] \,.
\end{align}

From (\ref{eq:prag_value_linear}) we know the first term is linear in the belief $b$.

The second term can be written as:
\begin{align}
\begin{split}
    &\mathbb{E}_{P(o'|b, a)}[\mathbb{KL}[b'(s'|o', b, a) || b'(s'|b, a)]] \\
    &= \mathbb{E}_{P(o', s'|b, a)}[\log b'(s'|o', b, a) - \log b'(s'|b, a)] \\
    &= \mathbb{E}_{P(o', s'|b, a)}[\log b'(s'|b, a) + \log P(o'|s') - \log P(o'|b, a) - \log b'(s'|b, a)] \\
    &= \mathbb{E}_{P(o', s'|b, a)}[\log P(o'|s') - \log P(o'|b, a)] \\
    &= \mathbb{H}[P(o'|b, a)] - \mathbb{E}_{P(s'|b, a)}[\mathbb{H}[P(o'|s')]] \\
    &= -\sum_{o'}P(o'|b, a)\log P(o'|b, a) - \sum_{s}b(s)\sum_{s'}P(s'|s, a)\mathbb{H}[P(o'|s')] \,.
\end{split}
\end{align}
The second term above is a linear function of the belief.

Applying the definition of convexity to the negative of the first term:
\begin{align}
\begin{split}
    &\sum_{o'}P(o'|\lambda b + (1 - \lambda)b', a)\log P(o'|\lambda b + (1 - \lambda)b', a) \\
    &= \sum_{o'}\sum_{s}P(o'|s, a)\left[\lambda b(s) + (1 - \lambda)b'(s)\right]\log\left[ \sum_{s}\left(\lambda b(s)P(o'|s, a) + (1 - \lambda) b'(s)P(o'|s, a))\right)\right] \\
    &= \sum_{o'}\left[\lambda P(o'|b, a) + (1 - \lambda) P(o'|b, a)\right]\log \frac{\lambda P(o'|b, a) + (1 - \lambda)P(o'|b', a)}{\lambda + (1 - \lambda)} \\
    &\leq \sum_{o'}\lambda P(o'|b, a)\log P(o'|b, a) + \sum_{o'}(1 - \lambda)P(o'|b', a)\log P(o'|b', a) \,,
\end{split}
\end{align}
where the last line uses the log sum inequality and shows the equation is convex. Thus, the first term is concave and the EFE reward is concave in the belief. 
\end{proof}

\subsection{Proofs for Section \ref{sec:pdmmdp}}

\begin{lemma}
(Performance difference in mismatched MDPs; restate of lemma \ref{lemma:pdmmdp}) Let $\pi$ and $\pi'$ be two policies which are optimal w.r.t. two MDPs $M$ and $M'$. The two MDPs share the same initial state distribution and discount factor but have different rewards $R, R'$ and dynamics $P, P'$. Denote $\Delta R(s, a) = R'(s, a) - R(s, a)$. The performance difference between $\pi$ and $\pi'$ when both are evaluated in $M$ is given by:
\begin{align}
\begin{split}
    &J_{M}(\pi) - J_{M}(\pi') \\
    &= \underbrace{\frac{1}{(1 - \gamma)}\mathbb{E}_{(s, a) \sim d^{\pi}_{P}}\left[A^{\pi'}_{M'}(s, a)\right]}_{\text{Advantage under expert distribution}} \\ 
    &\quad + \underbrace{\frac{1}{(1 - \gamma)}\mathbb{E}_{(s, a) \sim d^{\pi'}_{P}}\left[\Delta R(s, a) + \gamma\left(\mathbb{E}_{s' \sim P'(\cdot|s, a)}[V^{\pi'}_{M'}(s')] - \mathbb{E}_{s'' \sim P(\cdot|s, a)}[V^{\pi'}_{M'}(s'')]\right)\right]}_{\text{Reward-model advantage under own distribution}} \\
    &\quad + \underbrace{\frac{1}{(1 - \gamma)}\mathbb{E}_{(s, a) \sim d^{\pi}_{P}}\left[-\Delta R(s, a) + \gamma\left(\mathbb{E}_{s'' \sim P(\cdot|s, a)}[V^{\pi'}_{M'}(s'')] - \mathbb{E}_{s' \sim P'(\cdot|s, a)}[V^{\pi'}_{M'}(s')]\right)\right]}_{\text{Reward-model disadvantage under expert distribution}} \,.
\end{split}
\end{align}
\end{lemma}

\begin{proof}
Following \citep{vemula2023virtues}, we expand the performance difference as:
\begin{align}
\begin{split}
J_{M}(\pi) - J_{M}(\pi') &= \mathbb{E}_{\mu(s_0)}[V^{\pi}_{M}(s_0) - V^{\pi'}_{M}(s_0)] \\
&= \mathbb{E}_{\mu(s_0)}[V^{\pi}_{M}(s_0) - V^{\pi'}_{M'}(s_0)] + \mathbb{E}_{\mu(s_0)}[V^{\pi'}_{M'}(s_0) - V^{\pi'}_{M}(s_0)] \,.
\end{split}
\end{align}

The second term can be expanded as:
\begin{align}
\begin{split}
&\mathbb{E}_{\mu(s_{0})}[V^{\pi'}_{M'}(s_0) - V^{\pi'}_{M}(s_0)] \\
&= \mathbb{E}_{s_{0} \sim \mu(\cdot), a_0 \sim \pi'(\cdot|s_0)}[R'(s_0, a_0) + \gamma\mathbb{E}_{s_1 \sim P'(\cdot|s_0, a_0)}[V^{\pi'}_{M'}(s_1)] - R(s_0, a_0) - \gamma\mathbb{E}_{s_1 \sim P(\cdot|s_0, a_0)}[V^{\pi'}_{M}(s_1)]] \\
&= \mathbb{E}_{s_{0} \sim \mu(\cdot), a_0 \sim \pi'(\cdot|s_0)}[\Delta R(s_0, a_0) + \gamma\mathbb{E}_{s_1 \sim P'(\cdot|s_0, a_0)}[V^{\pi'}_{M'}(s_1)] - \gamma\mathbb{E}_{s_1 \sim P(\cdot|s_0, a_0)}[V^{\pi'}_{M'}(s_1)] \\
&\quad + \gamma\mathbb{E}_{s_1 \sim P(\cdot|s_0, a_0)}[V^{\pi'}_{M'}(s_1)] - \gamma\mathbb{E}_{s_1 \sim P(\cdot|s_0, a_0)}[V^{\pi'}_{M}(s_1)]] \\
&= \mathbb{E}_{s_{0} \sim \mu(\cdot), a_0 \sim \pi'(\cdot|s_0)}[\Delta R(s_0, a_0) + \gamma\mathbb{E}_{s_1 \sim P'(\cdot|s_0, a_0)}[V^{\pi'}_{M'}(s_1)] - \gamma\mathbb{E}_{s_1 \sim P(\cdot|s_0, a_0)}[V^{\pi'}_{M'}(s_1)]] \\
&\quad + \gamma\mathbb{E}_{s_{0} \sim \mu(\cdot), a_0 \sim \pi'(\cdot|s_0), s_{1} \sim P(\cdot|s_{0}, a_{0})}[\underbrace{V^{\pi'}_{M'}(s_1) - V^{\pi'}_{M}(s_1)}_{\text{term a}}] \,,
\end{split}
\end{align}
where $\Delta R(s, a) = R'(s, a) - R(s, a)$. 

Expanding term a, we arrive at a similar structure to the above:
\begin{align}
\begin{split}
    \text{term a} &= \mathbb{E}_{a_{1} \sim \pi'(\cdot|s_{1})}[R'(s_{1}, a_{1}) + \gamma\mathbb{E}_{s_2 \sim P'(\cdot|s_1, a_1)}[V^{\pi'}_{M'}(s_2)] - R(s_1, a_1) - \gamma\mathbb{E}_{s_2 \sim P(\cdot|s_1, a_1)}[V^{\pi'}_{M}(s_2)]] \\
    &= \mathbb{E}_{a_1 \sim \pi'(\cdot|s_1)}[\Delta R(s_1, a_1) + \gamma\mathbb{E}_{s_2 \sim P'(\cdot|s_1, a_1)}[V^{\pi'}_{M'}(s_2)] - \gamma\mathbb{E}_{s_2 \sim P(\cdot|s_1, a_1)}[V^{\pi'}_{M'}(s_2)]] \\
    &\quad + \gamma\mathbb{E}_{a_1 \sim \pi'(\cdot|s_1), s_{2} \sim P(\cdot|s_{1}, a_{1})}\underbrace{[V^{\pi'}_{M'}(s_2) - V^{\pi'}_{M}(s_2)]}_{\text{term a'}} \,.
\end{split}
\end{align}

We can thus unroll the last term iteratively and obtain:
\begin{align}
\begin{split}
    &\mathbb{E}_{\mu(s_{0})}[V^{\pi'}_{M'}(s_0) - V^{\pi'}_{M}(s_0)] \\
    &= \mathbb{E}\left[\sum_{t=0}^{\infty}\gamma^{t}\left(\Delta R(s_{t}, a_{t}) + \gamma\mathbb{E}_{s' \sim P'(\cdot|s_{t}, a_{t})}[V^{\pi'}_{M'}(s')] - \gamma\mathbb{E}_{s'' \sim P(\cdot|s_{t}, a_{t})}[V^{\pi'}_{M'}(s'')]\right)\right] \\
    &= \frac{1}{(1 - \gamma)}\mathbb{E}_{(s, a) \sim d^{\pi'}_{P}}\left[\Delta R(s, a) + \gamma\left(\mathbb{E}_{s' \sim P'(\cdot|s, a)}[V^{\pi'}_{M'}(s')] - \mathbb{E}_{s'' \sim P(\cdot|s, a)}[V^{\pi'}_{M'}(s'')]\right)\right] \,,
\end{split}
\end{align}
where the expectation in the second line is taken w.r.t. the stochastic process induced by $\pi', P$. 

We now expand the first term in the performance difference: 
\begin{align}
\begin{split}
&\mathbb{E}_{s_{0} \sim \mu(\cdot)}[V^{\pi}_{M}(s_0) - V^{\pi'}_{M'}(s_0)] \\
&= \bigg(\mathbb{E}_{s_{0} \sim \mu(\cdot)}[V^{\pi}_{M}(s_0) - \mathbb{E}_{a_0 \sim \pi(\cdot|s_0)}[Q^{\pi'}_{M'}(s_0, a_0)]]\bigg) + \bigg(\mathbb{E}_{s_{0} \sim \mu(\cdot)}[\mathbb{E}_{a_0 \sim \pi(\cdot|s_0)}[Q^{\pi'}_{M'}(s_0, a_0)] - V^{\pi'}_{M'}(s_0)]\bigg) \\
&= \bigg(\mathbb{E}_{s_{0} \sim \mu(\cdot), a_{0} \sim \pi(\cdot|s_{0})}[Q^{\pi'}_{M'}(s_0, a_0)] - V^{\pi'}_{M'}(s_0)]\bigg) \\
&\quad + \mathbb{E}_{s_{0} \sim \mu(\cdot), a_0 \sim \pi(\cdot|s_0)}[R(s_0, a_0) + \gamma \mathbb{E}_{s_1 \sim P(\cdot|s_0, a_0)}[V^{\pi}_{M}(s_1)]] \\
&\quad - \mathbb{E}_{s_{0} \sim \mu(\cdot), a_0 \sim \pi(\cdot|s_0)}[R'(s_0, a_0) + \gamma \mathbb{E}_{s_1 \sim P'(\cdot|s_0, a_0)}[V^{\pi'}_{M'}(s_1)]] \\
&= \mathbb{E}_{s_{0} \sim \mu(\cdot), a_0 \sim \pi(\cdot|s_0)}[A^{\pi'}_{M'}(s_0, a_0)] \\
&\quad + \mathbb{E}_{s_{0} \sim \mu(\cdot), a_0 \sim \pi(\cdot|s_0)}\bigg[-\Delta R(s_0, a_0) + \gamma\mathbb{E}_{s_1 \sim P(\cdot|s_0, a_0)}[V^{\pi}_{M}(s_1)] - \gamma\mathbb{E}_{s_1 \sim P(\cdot|s_0, a_0)}[V^{\pi'}_{M'}(s_1)] \\
&\quad + \gamma\mathbb{E}_{s_1 \sim P(\cdot|s_0, a_0)}[V^{\pi'}_{M'}(s_1)] - \gamma\mathbb{E}_{s_1 \sim P'(\cdot|s_0, a_0)}[V^{\pi'}_{M'}(s_1)]\bigg] \\
&= \mathbb{E}_{s_{0} \sim \mu(\cdot), a_0 \sim \pi(\cdot|s_0)}[A^{\pi'}_{M'}(s_0, a_0)] \\
&\quad + \mathbb{E}_{s_{0} \sim \mu(\cdot), a_0 \sim \pi(\cdot|s_0)}[-\Delta R(s_0, a_0) + \gamma\mathbb{E}_{s_1 \sim P(\cdot|s_0, a_0)}[V^{\pi'}_{M'}(s_1)] - \gamma\mathbb{E}_{s_1 \sim P'(\cdot|s_0, a_0)}[V^{\pi'}_{M'}(s_1)]] \\
&\quad + \gamma\mathbb{E}_{s_{0} \sim \mu(\cdot), a_0 \sim \pi(\cdot|s_0), s_1 \sim P(\cdot|s_0, a_0)}[\underbrace{V^{\pi}_{M}(s_1) - V^{\pi'}_{M'}(s_1)]}_{\text{term b}} \,.\\
\end{split}
\end{align}

Apply the same unrolling method to term b, we have:
\begin{align}
\begin{split}
    &\mathbb{E}_{s_{0} \sim \mu(\cdot)}[V^{\pi}_{M}(s_0) - V^{\pi'}_{M'}(s_0)] \\
    &= \mathbb{E}\left[\sum_{t=0}^{\infty}\gamma^{t}A^{\pi'}_{M'}(s_{t}, a_{t})\right]\\
    &\quad + \mathbb{E}\left[\sum_{t=0}^{\infty}\gamma^{t}\left(-\Delta R(s_{t}, a_{t}) + \gamma\mathbb{E}_{s'' \sim P(\cdot|s_{t}, a_{t})}[V^{\pi'}_{M'}(s'')] - \gamma\mathbb{E}_{s' \sim P'(\cdot|s_{t}, a_{t})}[V^{\pi'}_{M'}(s')]\right)\right] \\
    &= \frac{1}{(1 - \gamma)}\mathbb{E}_{(s, a) \sim d^{\pi}_{P}}\left[A^{\pi'}_{M'}(s, a)\right]\\ &\quad + \frac{1}{(1 - \gamma)}\mathbb{E}_{(s, a) \sim d^{\pi}_{P}}\left[-\Delta R(s, a) + \gamma\left(\mathbb{E}_{s'' \sim P(\cdot|s, a)}[V^{\pi'}_{M'}(s'')] - \mathbb{E}_{s' \sim P'(\cdot|s, a)}[V^{\pi'}_{M'}(s')]\right)\right] \,,
\end{split}
\end{align}
where the expectations in the first equality is again taken w.r.t. the stochastic process induced by $\pi, P$. 

Putting together, we have:
\begin{align}
\begin{split}
    &J_{M}(\pi) - J_{M}(\pi') \\
    &= \frac{1}{(1 - \gamma)}\mathbb{E}_{(s, a) \sim d^{\pi}_{P}}\left[A^{\pi'}_{M'}(s, a)\right]\\ 
    &\quad + \frac{1}{(1 - \gamma)}\mathbb{E}_{(s, a) \sim d^{\pi'}_{P}}\left[\Delta R(s, a) + \gamma\left(\mathbb{E}_{s' \sim P'(\cdot|s, a)}[V^{\pi'}_{M'}(s')] - \mathbb{E}_{s'' \sim P(\cdot|s, a)}[V^{\pi'}_{M'}(s'')]\right)\right] \\
    &\quad + \frac{1}{(1 - \gamma)}\mathbb{E}_{(s, a) \sim d^{\pi}_{P}}\left[-\Delta R(s, a) + \gamma\left(\mathbb{E}_{s'' \sim P(\cdot|s, a)}[V^{\pi'}_{M'}(s'')] - \mathbb{E}_{s' \sim P'(\cdot|s, a)}[V^{\pi'}_{M'}(s')]\right)\right] \,.
\end{split}
\end{align}
\end{proof}

\begin{proposition}\label{prop:model_advantage_bound}
(Model advantage bound) Let $V^{\pi}_{P}(s)$ be the value function of policy $\pi$ in dynamics $P$ with reward $R(s, a)$ and let $R_{max} = \max_{s, a}|R(s, a)|$. The absolute value of the advantage of dynamics $P$ over $P'$ is bounded by:
\begin{align}
    \left|\mathbb{E}_{P(s'|s, a)}[V^{\pi}_{P}(s')] - \mathbb{E}_{P'(s''|s, a)}[V^{\pi}_{P}(s'')]\right| \leq \frac{R_{max}}{1 - \gamma}\sqrt{2\mathbb{KL}[P(\cdot|s, a) || P'(\cdot|s, a)]} \,.
\end{align}
\end{proposition}

\begin{proof}
The proof is the same as lemma B.2. in \citep{wei2023bayesian}.
\begin{align}
\begin{split}
    &\left|\mathbb{E}_{P(s'|s, a)}[V^{\pi}_{P}(s')] - \mathbb{E}_{P'(s''|s, a)}[V^{\pi}_{P}(s'')]\right| \\
    &= \left|\sum_{s'}V^{\pi}_{P}(s')\left(P(s'|s, a) - P'(s'|s, a)\right)\right| \\
    &\overset{(1)}{\leq} \sum_{s'}\left|V^{\pi}_{P}(s')\right|\left|P(s'|s, a) - P'(s'|s, a)\right| \\
    &\overset{(2)}{\leq} \Vert V^{\pi}_{P}(\cdot) \Vert_{\infty}\Vert P(\cdot|s, a) - P'(\cdot|s, a) \Vert_{1} \\
    &\overset{(3)}{\leq} \Vert V^{\pi}_{P}(\cdot) \Vert_{\infty}\sqrt{2\mathbb{KL}[P(\cdot|s, a) || P'(\cdot|s, a)]} \,,
\end{split}
\end{align}
where (1) uses Jensen's inequality since the inner sum is a convex combination, (2) uses Holder's inequality and (3) uses Pinsker's inequality. The coefficient $\Vert V^{\pi}_{P}(\cdot) \Vert_{\infty} \leq \mathbb{E}_{\pi, P}[\sum_{t=0}^{\infty}\gamma^{t}\max_{s, a}|R(s, a)|] = \frac{R_{max}}{(1 - \gamma)}$. 

Putting together, we have:
\begin{align}
    \left|\mathbb{E}_{P(s'|s, a)}[V^{\pi}_{P}(s')] - \mathbb{E}_{P'(s''|s, a)}[V^{\pi}_{P}(s'')]\right| \leq \frac{R_{max}}{1 - \gamma}\sqrt{2\mathbb{KL}[P(\cdot|s, a) || P'(\cdot|s, a)]} \,.
\end{align}

\end{proof}

\begin{lemma}
(Restate of lemma \ref{lemma:pdmmdp_bound}) For the setting considered in lemma \ref{lemma:pdmmdp}, let $\epsilon_{\pi'} = \mathbb{E}_{(s, a) \sim d^{\pi}_{P}}[|A^{\pi'}_{M'}(s, a)|]$, $\epsilon_{R'} = \mathbb{E}_{(s, a) \sim d^{\pi}_{P}}[|\Delta R(s, a)|]$, $\epsilon_{P'} = \mathbb{E}_{(s, a) \sim d^{\pi}_{P}}[\mathbb{KL}[P(\cdot|s, a) || P'(\cdot|s, a)]]$, and $R'_{max} = \max_{s, a}|R'(s, a)|$. Let the two policies have bounded state-action marginal density ratio $\frac{d^{\pi'}_{P}(s, a)}{d^{\pi}_{P}(s, a)} \leq C$. The performance gap is bounded as:
\begin{align}
    J_{M}(\pi) - J_{M}(\pi') \leq \frac{1}{1 - \gamma}\epsilon_{\pi'} + \frac{C + 1}{1 - \gamma }\epsilon_{R'} + \frac{(C + 1)\gamma R'_{max}}{(1 - \gamma)^{2}}\sqrt{2\epsilon_{P'}} \,.
\end{align}
\end{lemma}

\begin{proof}
The absolute value of the performance gap can be written as:
\begin{align}
\begin{split}
    &|J_{M}(\pi) - J_{M}(\pi')| \\
    &\leq \frac{1}{(1 - \gamma)}\mathbb{E}_{(s, a) \sim d^{\pi}_{P}}\left[|A^{\pi'}_{M'}(s, a)|\right]\\ 
    &\quad + \frac{1}{(1 - \gamma)}\left|\mathbb{E}_{(s, a) \sim d^{\pi}_{P}}\left[-\Delta R(s, a)\right]\right| + \frac{1}{(1 - \gamma)}\left|\mathbb{E}_{(s, a) \sim d^{\pi'}_{P}}\left[\Delta R(s, a)\right]\right| \\
    &\quad + \frac{\gamma}{(1 - \gamma)}\left|\mathbb{E}_{(s, a) \sim d^{\pi}_{P}}\left[\mathbb{E}_{s'' \sim P(\cdot|s, a)}[V^{\pi'}_{M'}(s'')] - \mathbb{E}_{s' \sim P'(\cdot|s, a)}[V^{\pi'}_{M'}(s')]\right]\right| \\
    &\quad + \frac{\gamma}{(1 - \gamma)}\left|\mathbb{E}_{(s, a) \sim d^{\pi'}_{P}}\left[\mathbb{E}_{s' \sim P'(\cdot|s, a)}[V^{\pi'}_{M'}(s')] - \mathbb{E}_{s'' \sim P(\cdot|s, a)}[V^{\pi'}_{M'}(s'')]\right]\right|
\end{split}
\end{align}
due to Jensen's inequality.

Expanding the third term reward advantage on the right hand side:
\begin{align}
\begin{split}
    \left|\mathbb{E}_{(s, a) \sim d^{\pi'}_{P}}\left[\Delta R(s, a)\right]\right| &= \left|\mathbb{E}_{(s, a) \sim d^{\pi}_{P}}\left[\frac{d^{\pi'}_{P}(s, a)}{d^{\pi}_{P}(s, a)}\Delta R(s, a)\right]\right| \\
    &\leq \mathbb{E}_{(s, a) \sim d^{\pi}_{P}}\left[\left|\frac{d^{\pi'}_{P}(s, a)}{d^{\pi}_{P}(s, a)}\Delta R(s, a)\right|\right] \\
    &\leq \left\Vert \frac{d^{\pi'}_{P}(s, a)}{d^{\pi}_{P}(s, a)}\right\Vert_{\infty}\mathbb{E}_{(s, a) \sim d^{\pi}_{P}}[|\Delta R(s, a)|] \\
    &= C \epsilon_{R'} \,.
\end{split}
\end{align}
For the second term we drop $C$ from the above due to no distribution mismatch. 

Applying proposition \ref{prop:model_advantage_bound} to the last term:
\begin{align}
\begin{split}
    &\left|\mathbb{E}_{(s, a) \sim d^{\pi'}_{P}}\left[\mathbb{E}_{s' \sim P'(\cdot|s, a)}[V^{\pi'}_{M'}(s')] - \mathbb{E}_{s'' \sim P(\cdot|s, a)}[V^{\pi'}_{M'}(s'')]\right]\right| \\
    &= \left|\mathbb{E}_{(s, a) \sim d^{\pi}_{P}}\left[\frac{d^{\pi'}_{P}(s, a)}{d^{\pi}_{P}(s, a)}\left(\mathbb{E}_{s' \sim P'(\cdot|s, a)}[V^{\pi'}_{M'}(s')] - \mathbb{E}_{s'' \sim P(\cdot|s, a)}[V^{\pi'}_{M'}(s'')]\right)\right]\right| \\
    &\leq \mathbb{E}_{(s, a) \sim d^{\pi}_{P}}\left[\left|\frac{d^{\pi'}_{P}(s, a)}{d^{\pi}_{P}(s, a)}\left(\mathbb{E}_{s' \sim P'(\cdot|s, a)}[V^{\pi'}_{M'}(s')] - \mathbb{E}_{s'' \sim P(\cdot|s, a)}[V^{\pi'}_{M'}(s'')]\right)\right|\right] \\
    &\leq \left\Vert\frac{d^{\pi'}_{P}(s, a)}{d^{\pi}_{P}(s, a)}\right\Vert_{\infty}\Vert V^{\pi'}_{M'}(\cdot)\Vert_{\infty}\mathbb{E}_{(s, a) \sim d^{\pi}_{P}}[\Vert P(\cdot|s, a) - P'(\cdot|s, a)\Vert_{1}] \\
    &\leq \left\Vert\frac{d^{\pi'}_{P'}(s, a)}{d^{\pi}_{P}(s, a)}\right\Vert_{\infty}\Vert V^{\pi'}_{M'}(\cdot)\Vert_{\infty}\sqrt{2\mathbb{E}_{(s, a) \sim d^{\pi}_{P}}[\mathbb{KL}[P(\cdot|s, a) || P'(\cdot|s, a)]]} \\
    &= \frac{C R'_{max}}{1 - \gamma}\sqrt{2\epsilon_{P'}} \,.
\end{split}
\end{align}
Again for the fourth term we drop $C$ from the above due to no distribution mismatch. 

Putting together and apply the fact that $J_{M}(\pi) \geq J_{M}(\pi')$, we have:
\begin{align}
    J_{M}(\pi) - J_{M}(\pi') \leq \frac{1}{1 - \gamma}\epsilon_{\pi'} + \frac{C + 1}{1 - \gamma }\epsilon_{R'} + \frac{(C + 1)\gamma R'_{max}}{(1 - \gamma)^{2}}\sqrt{2\epsilon_{P'}} \,.
\end{align}

\end{proof}

\subsection{Proofs for Section \ref{sec:voi_rl}}

\subsubsection{Helpful Identities}

\begin{proposition}
(Open-loop value function convexity) The open-loop value function as defined in (\ref{eq:belief_mdp_open}) is piece-wise linear and convex in the beliefs.
\end{proposition}
\begin{proof}
Recall the definition of the open-loop value function is:
\begin{align}
    Q^{open}(b, a) = \sum_{s}b(s)R(s, a) + \gamma V^{open}(b'(a, b)) \,.
\end{align}
Furthermore, it is a valid belief MDP given the deterministic transition of the belief state defined in (\ref{eq:efe_belief_mdp_dynamics}). 

Although this is an infinite horizon value function, due to the contraction mapping property of Bellman equation \citep{agarwal2019reinforcement}, it can be approximated arbitrarily close using a finite number of $K$ iterations starting from the base case $Q^{open}_{k=0}(b, a) = \sum_{s}b(s)R(s, a)$. It is clear the base case value function $V^{open}_{k=0}(b) = \max_{\tilde{a}}Q^{open}_{k=0}(b, \tilde{a})$ is piecewise linear and convex in $b$. 

For iteration $k \in \{1, ..., \infty\}$, we have:
\begin{align}
    Q^{open}_{k+1}(b, a) = \sum_{s}b(s)R(s, a) + \gamma \max_{a'}Q^{open}_{k}(b'(a, b), a') \,.
\end{align}
The belief update $b'(a, b) = \sum_{s}P(s'|s, a)b(s)$ is linear and convex in $b$, making the second term piecewise linear and convex. The first term is also linear and convex. The combination is thus piecewise linear and convex.
\end{proof}


\begin{proposition}\label{prop:evpo}
(EVPO non-negativity) Let the expected value of perfect observation for a single stage decision making problem with reward $R(s, a)$, prior belief $b(s)$ and marginal observation distribution $P(o) = \sum_{s}P(o|s)b(s)$ be defined as:
\begin{align}\label{eq:evpo}
\begin{split}
    EVPO &= EV|PO - EV \,,\\
    EV &= \max_{a} \sum_{s}b(s)R(s, a) \,,\\
    EV|PO &= \sum_{p}P(o)\max_{a}\sum_{s}b(s|o)R(s, a) \,.
\end{split}
\end{align}
It holds that $EVPO \geq 0$.
\end{proposition}

\begin{proof}
We wish to show:
\begin{align}
\sum_{o}P(o)\max_{a}\sum_{s}b(s|o)R(s, a) \geq \max_{a'}\sum_{s}b(s)R(s, a') \,.
\end{align}

Let use define $a^{*}(o) = \arg\max_{a}\sum_{s}b(s|o)R(s, a)$, and $a^{*} = \arg\max_{a}\sum_{s}b(s)R(s, a)$ so that we can write the LHS as $\sum_{o}P(o)\sum_{s}b(s|o)R(s, a^{*}(o))$ and the RHS as $\sum_{s}b(s)R(s, a^{*})$.

By definition, we have:
\begin{align}
\sum_{s}b(s|o)R(s, a^{*}(o)) &\geq \sum_{s}b(s|o)R(s, a^{*}) \,,
\end{align}
since $a^{*}(o)$ is the optimal action taking into consideration of $o$.

Applying expectation over $P(o)$ to the above inequality, we have:
\begin{align}
\begin{split}
\sum_{o}P(o)\sum_{s}b(s|o)R(s, a^{*}(o)) &\geq \sum_{o}P(o)\sum_{s}b(s|o)R(s, a^{*}) \\
&= \sum_{s}b(s)R(s, a^{*}) \,,
\end{split}
\end{align}
which completes the proof.
\end{proof}

\begin{proposition}\label{prop:evpo_upperbound}
(EVPO upper bound) Let $R_{max} = \max_{s, a}|R(s, a)|$. The expected value of perfect observation as defined in (\ref{eq:evpo}) is upper bounded as follows:
\begin{align}
    EVPO \leq R_{max}\sqrt{2 \mathbb{E}_{P(o)}[\mathbb{KL}[b(s|o) || b(s)]]} \,.
\end{align}
\end{proposition}

\begin{proof}
Recall the definition of EVPO is:
\begin{align}
\begin{split}
    EVPO &= \mathbb{E}_{P(o)}[V(b(s|o))] - V(b(s)) \\
    &= \mathbb{E}_{P(o)}\left[\max_{a(o)}\sum_{s}b(s|o)R(s, a(o))\right] - \max_{a}\sum_{s}b(s)R(s, a) \\
    &\leq \mathbb{E}_{P(o)}\left[\sum_{s}b(s|o)R(s, a^{*}(o))\right] - \sum_{s}b(s)R(s, a^{*}(o)) \\
    &= \mathbb{E}_{P(o)}\left[\sum_{s}R(s, a^{*}(o))\left(b(s|o) - b(s)\right)\right] \,,
\end{split}
\end{align}
where we have used $a^{*}(o) = \arg\max_{a(o)}\sum_{s}b(s|o)R(s, a(o))$ and the inequality is due to $a^{*}(o)$ being suboptimal for the second term.

Taking the absolute value of the above EVPO bound, we have:
\begin{align}
\begin{split}
    |EVPO| &= \left|\mathbb{E}_{P(o)}\left[\sum_{s}R(s, a^{*}(o))\left(b(s|o) - b(s)\right)\right]\right| \\
    &\overset{(1)}{\leq} \mathbb{E}_{P(o)}\left[\left|\sum_{s}R(s, a^{*}(o))\left(b(s|o) - b(s)\right)\right|\right] \\
    &\overset{(2)}{\leq} \mathbb{E}_{P(o)}\left[\sum_{s}|R(s, a^{*}(o))|\left|b(s|o) - b(s)\right|\right] \\
    &\overset{(3)}{\leq} \Vert R(\cdot, \cdot)\Vert_{\infty}\mathbb{E}_{P(o)}\left[\Vert b(s|o) - b(s)\Vert_{1}\right] \\
    &\overset{(4)}{\leq} R_{max}\sqrt{2\mathbb{E}_{P(o)}[\mathbb{KL}[b(s|o) || b(s)]]}
\end{split}
\end{align}
where (1) and (2) are due to Jensen's inequality, (3) is due to Holder's inequality, and (4) is due to Pinsker's inequality.
\end{proof}

\subsubsection{Main Results of Section \ref{sec:voi_rl}}

\begin{proposition}
(EVPO-POMDP non-negativity; restate of proposition \ref{prop:evpo_pomdp_nonnegative}) Let $Q^{open}(b, a), V^{open}(b)$ and $Q(b, a), V(b)$ denote the open and closed-loop value functions as defined in (\ref{eq:belief_mdp_open_close}), it holds that:
\begin{align}
    Q(b, a) \geq Q^{open}(b, a) \text{ and } V(b) \geq V^{open}(b) \text{ for all $b \in \Delta(\mathcal{S})$ and $a \in \mathcal{A}$} \,.
\end{align}
\end{proposition}

\begin{proof}
Recall the open and closed-loop value functions are defined as:
\begin{align}
\begin{split}
    Q^{open}(b, a) &= \sum_{s}b(s)R(s, a) + \gamma V^{open}(b'(a, b)), \quad V^{open}(b) = \max_{a}Q^{open}(b, a) \,, \\
    Q(b, a) &= \sum_{s}b(s)R(s, a) + \gamma \sum_{o'}P(o'|b, a)V(b'(o', a, b)), \quad V(b) = \max_{a}Q(b, a) \,.
\end{split}
\end{align}
Although these are infinite horizon value functions, again due to their contraction mapping property \citep{agarwal2019reinforcement}, they can be approximated arbitrarily close using a finite number of 
$K$ iterations starting from the base case $Q_{k=0}(b, a) = \sum_{s}b(s)R(s, a)$. 

Starting with $k=1$, we have:
\begin{align}
\begin{split}
    Q^{open}_{1}(b, a)&= \sum_{s}b(s)R(s, a) + \gamma V^{open}_{0}(b'(a, b)), \quad V^{open}_{0}(b) = \max_{a}\sum_{s}b(s)R(s, a) \,, \\
    Q_{1}(b, a) &= \sum_{s}b(s)R(s, a) + \gamma \sum_{o'}P(o'|b, a)V_{0}(b'(o', a, b)), \quad V_{0}(b) = \max_{a}\sum_{s}b(s)R(s, a) \,.
\end{split}
\end{align}

Taking the difference between the two value functions and multiply by $\frac{1}{\gamma}$, we have:
\begin{align}\label{eq:voi_proof_ev_diff}
\begin{split}
    &\frac{1}{\gamma}\left[Q_{1}(b, a) - Q^{open}_{1}(b, a) \right] \\
    &= \sum_{o'}P(o'|b, a)V_{0}(b'(o', a, b)) - V^{open}_{0}(b'(a, b)) \\
    &= \sum_{o'}P(o'|b, a)\max_{a^{close}}\sum_{s}b'(s'|o', a, b)R(s', a^{close}) - \max_{a^{open}}\sum_{s}b'(s'|a, b)R(s', a^{open}) \\
    &= EVPO \geq 0 \,,
\end{split}
\end{align}
where the second to last line equals EVPO in proposition \ref{prop:evpo} under prior belief $b'(s'|b, a)$ for all $b \in \Delta(s), a \in \mathcal{A}$. Thus it must be non-negative. 

Applying the above to the value functions at $k=1$, we have:
\begin{align}
\begin{split}
    V_{1}(b) - V^{open}_{1}(b) &= \max_{a^{close}}Q_{1}(b, a^{close}) - \max_{a^{open}}Q^{open}_{1}(b, a^{open}) \\
    &\geq Q_{1}(b, a^{open*}) - Q^{open}_{1}(b, a^{open*}) \\
    &\geq 0 \,,
\end{split}
\end{align}
where we have defined $a^{open*} = \arg\max_{a^{open}}Q^{open}_{1}(b, a^{open})$.

Now consider $k=2$, where
\begin{align}
\begin{split}
    Q^{open}_{2}(b, a)&= \sum_{s}b(s)R(s, a) + \gamma V^{open}_{1}(b'(a, b)), \quad V^{open}_{1}(b) = \max_{a}Q^{open}_{1}(s, a) \,, \\
    Q_{1}(b, a) &= \sum_{s}b(s)R(s, a) + \gamma \sum_{o'}P(o'|b, a)V_{1}(b'(o', a, b)), \quad V_{1}(b) = \max_{a}Q_{1}(s, a) \,.
\end{split}
\end{align}

Taking the difference between the two value functions again, we have:
\begin{align}
\begin{split}
    &\frac{1}{\gamma}\left[Q_{2}(b, a) - Q^{open}_{2}(b, a) \right] \\
    &= \sum_{o'}P(o'|b, a)V_{1}(b'(o, a, b)) - V^{open}_{1}(b'(a, b)) \\
    &= \sum_{o'}P(o'|b, a)\max_{a^{'close}}\left\{\sum_{s}b'(s|o', a, b)R(s, a^{'close}) + \sum_{o''}P(o''|b', a^{'close})V_{0}(b''(o'', a^{'close}, b'))\right\} \\
    &\quad - \max_{a^{'open}}\left\{\sum_{s}b'(s|a, b)R(s, a^{'open}) + V^{open}_{0}(b''(a^{'open}, b'))\right\} \,.
\end{split}
\end{align}

Let $a^{'open*} = \arg\max_{a^{'open}}\left\{\sum_{s}b'(s|a, b)R(s, a^{'open}) + V^{open}_{0}(b''(a^{'open}, b'))\right\}$ and $a^{'close*} = \arg\max_{a^{'close}}\left\{\sum_{s}b'(s|o', a, b)R(s, a^{'close}) + \sum_{o''}P(o''|b', a^{'close})V_{0}(b''(o'', a^{'close}, b'))\right\}$, we have:
\begin{align}
\begin{split}
    &\sum_{o'}P(o'|b, a)\max_{a^{'close}}\left\{\sum_{s}b'(s|o', a, b)R(s, a^{'close}) + \sum_{o''}P(o''|b', a^{'close})V_{0}(b''(o'', a^{'close}, b'))\right\} \\
    &\quad - \max_{a^{'open}}\left\{\sum_{s}b'(s|a, b)R(s, a^{'open}) + V^{open}_{0}(b''(a^{'open}, b'))\right\} \\
    &\geq \sum_{o'}P(o'|b, a)\max_{a^{'close}}\left\{\sum_{s}b'(s|o', a, b)R(s, a^{'close}) + \sum_{o''}P(o''|b', a^{'open*})V_{0}(b''(o'', a^{'open*}, b'))\right\} \\
    &\quad - \left\{\sum_{s}b'(s|a, b)R(s, a^{'open*}) + V^{open}_{0}(b''(a^{'open*}, b'))\right\} \\
    &= \underbrace{\sum_{o'}P(o'|b, a)\left\{\max_{a^{'close}}\sum_{s}b'(s|o', a, b)R(s, a^{'close}) - \sum_{s}b'(s|a, b)R(s, a^{'open*})\right\}}_{EVPO \geq 0} \\ 
    &\quad + \sum_{o'}P(o'|b, a)\underbrace{\left\{\sum_{o''}P(o''|b', a^{'open*})V_{0}(b''(o'', a^{'open*}, b')) - V^{open}_{0}(b''(a^{'open*}, b'))\right\}}_{\geq 0 \text{ due to (\ref{eq:voi_proof_ev_diff})}} \\
    &\geq 0 \,.
\end{split}
\end{align}

Applying the above to $k \in \{1, ..., \infty\}$ recursively, we have:
\begin{align}
    Q(b, a) \geq Q^{open}(b, a) \text{ and } V(b) \geq V^{open}(b) \,.
\end{align}
\end{proof}

\begin{proposition}\label{prop:closed_loop_model_advantage_appx}
(Closed-loop model advantage upper bound) Let $R_{max} = \max_{s, a}|R(s, a)|$. The closed-loop model advantage is upper bounded as follows:
\begin{align}
\begin{split}
    &\mathbb{E}_{P(b'|b, a)}[V^{open}(b')] - \mathbb{E}_{P^{open}(b''|b, a)}[V^{open}(b'')] \leq \frac{R_{max}}{1 - \gamma}\sqrt{2 IG(b, a)} \,.\\
\end{split}
\end{align}
\end{proposition}

\begin{proof}
Recall the closed-loop model advantage is defined as:
\begin{align}
\begin{split}
    &\mathbb{E}_{P(b'|b, a)}[V(b')] - \mathbb{E}_{P^{open}(b''|b, a)}[V(b'')] = \mathbb{E}_{P(o'|b, a)}[V(b'(s'|o', b, a))] - V(b'(s'))
\end{split}
\end{align}
To simplify notation, we will drop the conditioning on $b, a$ in the expectation. This also enables us to remove the "$'$" notation.

We will use a similar method as before where we leverage the contraction mapping property of the value function and start from the base case. It is clear for the base case $k=0$ where $V(b) = \max_{a}\sum_{s}b(s)R(s, a)$, the model advantage is EVPO and thus the upper bound from proposition \ref{prop:evpo_upperbound} applies. To simplify notation, let's denote the upper bound as $C(b)$ since $b(s|o)$ can be calculated from $b(s)$

We now consider $k=1$:
\begin{align}
\begin{split}
    &\mathbb{E}_{P(o)}[V_{1}(b(s|o))] - V_{1}(b(s)) \\
    &= \mathbb{E}_{P(o)}\left[\max_{a^{close}}\sum_{s}b(s|o)R(s, a^{close}) + \gamma V_{0}(b'(a^{close}, b(s|o)))\right] \\
    &\quad - \left[\max_{a^{open}}\sum_{s}b(s)R(s, a^{open}) + \gamma V_{0}(b'(a^{open}, b(s)))\right] \\
    &\leq \mathbb{E}_{P(o)}\left[\sum_{s}b(s|o)R(s, a^{close*}) + \gamma V_{0}(b'(a^{close*}, b(s|o)))\right] \\
    &\quad - \left[\sum_{s}b(s)R(s, a^{close*}) + \gamma V_{0}(b'(a^{close*}, b(s)))\right] \\
    &= \underbrace{\mathbb{E}_{P(o)}\left[\sum_{s}b(s|o)R(s, a^{close*}) - \sum_{s}b(s)R(s, a^{close*})\right]}_{\text{term a}} \\
    &\quad + \gamma \underbrace{\mathbb{E}_{P(o)}\left[V_{0}(b'(a^{close*}, b(s|o))) - V_{0}(b'(a^{close*}, b(s)))\right]}_{\text{term b}} \,.
\end{split}
\end{align}

Term a is the same as the one in EVPO, thus the upper bound $C(b)$ applies again. In term b, recall the open-loop belief updates are defined as:
\begin{align}
\begin{split}
    b'(a, b(s|o)) &= \sum_{s}P(s'|s, a)b(s|o) := b'(s'|o) \,, \\
    b'(a, b(s)) &= \sum_{s}P(s'|s, a)b(s) := b'(s') \,.
\end{split}
\end{align}
Due to the convexity of the value functions,  we have $\text{term b} \geq 0$. Furthermore, term b corresponds to EVPO for stage 0 with modified belief updates as defined above. Thus $C(b')$ applies again.

Combining both, we have:
\begin{align}
\begin{split}
    &\mathbb{E}_{P(o)}[V_{1}(b(s|o))] - V_{1}(b(s)) \\
    &\leq R_{max}\sqrt{2 \mathbb{E}_{P(o)}[\mathbb{KL}[b(s|o) || b(s)]]} + \gamma R_{max}\sqrt{2 \mathbb{E}_{P(o)}[\mathbb{KL}[b'(s'|o, a^{close*}) || b'(s')]]} \\
    &\leq R_{max}\sqrt{2 \mathbb{E}_{P(o)}[\mathbb{KL}[b(s|o) || b(s)]]} + \gamma R_{max}\sqrt{2 \mathbb{E}_{P(o)}[\mathbb{KL}[b(s|o) || b(s)]]} \,,
\end{split}
\end{align}
where the second inequality is due to data processing inequality.

Applying the above to $k \in \{2, ..., \infty\}$ recursively, we have:
\begin{align}
\begin{split}
    \mathbb{E}_{P(o'|b, a)}[V(b'(s'|o'))] - V(b'(s')) 
    &\leq R_{max}\sum_{t=0}^{\infty}\gamma^{t}\sqrt{2 \mathbb{E}_{P(o'|b, a)}[\mathbb{KL}[b'(s'|o') || b'(s')]]} \\
    &= \frac{R_{max}}{1 - \gamma}\sqrt{2 \mathbb{E}_{P(o'|b, a)}[\mathbb{KL}[b'(s'|o') || b'(s')]]} \,.
\end{split}
\end{align}
\end{proof}

\subsection{Proofs for Section \ref{sec:voi_efe}}

\begin{proposition}
\label{prop:evpo_upperbound_efe}
(EFE EVPO upper bound) Let $R_{max} = \max_{s, a}|R(s, a)|$. The expected value of perfect observation as defined in (\ref{eq:evpo}) is upper bounded as follows:
\begin{align}
    EVPO^{EFE} \leq \tilde{R}_{max}\sqrt{2 \mathbb{E}_{P(o)}[\mathbb{KL}[b(s|o) || b(s)]]} \,.
\end{align}
\end{proposition}

\begin{proof}
Recall the one-step EFE belief reward is:
\begin{align}
    R(b, a) = \sum_{s}b(s)R(s, a) + IG(b, a) \,,
\end{align}
where the reward is defined as $R(s, a) := \tilde{R}(s, a)$ in (\ref{eq:efe_belief_mdp_rwd}) and $IG(b, a)$ is the information gain.

We can thus write EVPO as:
\begin{align}
\begin{split}
    &EVPO \\
    &= \mathbb{E}_{P(o)}\left[\max_{a(o)}\sum_{s}b(s|o)R(s, a(o)) + IG(b(s|o), a(o))\right] - \max_{a}\left[\sum_{s}b(s)R(s, a) + IG(b(s), a)\right] \\
    &\leq \mathbb{E}_{P(o)}\left[\sum_{s}b(s|o)R(s, a^{*}(o)) + IG(b(s|o), a^{*}(o))\right] - \left[\sum_{s}b(s)R(s, a^{*}(o)) + IG(b(s), a^{*}(o))\right] \\
    &= \mathbb{E}_{P(o)}\left[\sum_{s}R(s, a^{*}(o))\left(b(s|o) - b(s)\right)\right] + \underbrace{\mathbb{E}_{P(o)}[IG(b(s|o), a^{*}(o)) - IG(b(s), a^{*}(o))]}_{\leq 0} \\
    &\leq \mathbb{E}_{P(o)}\left[\sum_{s}R(s, a^{*}(o))\left(b(s|o) - b(s)\right)\right] \,,
\end{split}
\end{align}
where we have used $a^{*}(o) = \arg\max_{a(o)}\sum_{s}b(s|o)R(s, a(o))$ and the last inequality is due to $IG$ being a concave function of beliefs. The remaining term is the same as the one in proposition \ref{prop:evpo_upperbound}. Thus, applying the result from proposition \ref{prop:evpo_upperbound} we complete the proof.
\end{proof}

\begin{proposition}\label{prop:efe_closed_loop_model_advantage_appx}
(EFE closed-loop model advantage upper bound) Let $R_{max} = \max_{s, a}|R(s, a)|$. The closed-loop model advantage under the EFE value function is upper bounded as follows:
\begin{align}
\begin{split}
    &\mathbb{E}_{P(b'|b, a)}[V^{EFE}(b')] - \mathbb{E}_{P^{open}(b''|b, a)}[V^{EFE}(b'')] \leq \frac{R_{max}}{1 - \gamma}\sqrt{2 IG(b, a)}\\
\end{split}
\end{align}
\end{proposition}

\begin{proof}
Similar to the proof to proposition \ref{prop:closed_loop_model_advantage_appx}, we start with the base case which is covered by proposition \ref{prop:evpo_upperbound_efe}. To simplify notation, we drop the EFE superscript with the understanding that $V^{EFE}$ is the value function under the EFE belief MDP. 

Starting with $k=1$, we have:
\begin{align}
\begin{split}
    &\mathbb{E}_{P(o)}[V_{1}(b(s|o))] - V_{1}(b(s)) \\
    &= \mathbb{E}_{P(o)}\left[\max_{a^{close}}\sum_{s}b(s|o)R(s, a^{close}) + IG(b(s|o), a^{close}) + \gamma V_{0}(b'(a^{close}, b(s|o)))\right] \\
    &\quad - \left[\max_{a^{open}}\sum_{s}b(s)R(s, a^{open}) + IG(b(s), a^{open}) + \gamma V_{0}(b'(a^{open}, b(s)))\right] \\
    &\leq \mathbb{E}_{P(o)}\left[\sum_{s}b(s|o)R(s, a^{close*}) + IG(b(s|o), a^{close*}) + \gamma V_{0}(b'(a^{close*}, b(s|o)))\right] \\
    &\quad - \left[\sum_{s}b(s)R(s, a^{close*}) + IG(b(s), a^{close*}) + \gamma V_{0}(b'(a^{close*}, b(s)))\right] \\
    &= \mathbb{E}_{P(o)}\left[\sum_{s}b(s|o)R(s, a^{close*}) - \sum_{s}b(s)R(s, a^{close*})\right] \\
    &\quad + \underbrace{\mathbb{E}_{P(o)}[IG(b(s|o), a^{close*}) - IG(b(s), a^{close*})]}_{\leq 0} \\
    &\quad + \gamma \mathbb{E}_{P(o)}\left[V_{0}(b'(a^{close*}, b(s|o))) - V_{0}(b'(a^{close*}, b(s)))\right] \\
    &\leq \underbrace{\mathbb{E}_{P(o)}\left[\sum_{s}b(s|o)R(s, a^{close*}) - \sum_{s}b(s)R(s, a^{close*})\right]}_{\text{term a}} \\
    &\quad + \gamma \underbrace{\mathbb{E}_{P(o)}\left[V_{0}(b'(a^{close*}, b(s|o))) - V_{0}(b'(a^{close*}, b(s)))\right]}_{\text{term b}}
\end{split}
\end{align}
We arrive at the same form as proposition \ref{prop:closed_loop_model_advantage_appx}. While we cannot guarantee term b $> 0$, the same upper bound holds. The next remark ensures the expected closed-loop model advantage under the EFE reward is non-negative, which provides the motivation for assumption \ref{assumption:reward_specification}. 

Finally, applying the above recursively to $k \in \{2, ..., \infty\}$, we complete the proof.
\end{proof}

\begin{remark}
(Motivation for assumption \ref{assumption:reward_specification}) To ensure the EFE model advantage expected under the Bayes optimal policy $\pi$ is non-negative, we need to set the reward such that:
\begin{align}
\begin{split}
    \mathbb{E}_{(b, a) \sim d^{\pi}_{P}}\left[\sum_{s}\left(b(s|o) - b(s)\right)R(s, a)\right] \geq \mathbb{E}_{(b, a) \sim d^{\pi}_{P}}[IG(b(s), a) - IG(b(s|o), a)] \,,
\end{split}
\end{align}
where $d^{\pi}_{P}$ is the marginal distribution induced by the Bayes optimal policy in the closed-loop belief dynamics. 
\end{remark}

\begin{theorem}\label{theorem:open_efe_performance_gap_appx}
(Open-loop and EFE policy performance gaps; restate of theorem \ref{theorem:open_efe_performance_gap}) Let all policies be deployed in POMDP $M$ and all are allowed to update their beliefs according to $b'(o', a, b)$. Let $\epsilon_{IG} = \mathbb{E}_{(b, a) \sim d^{\pi}_{P}}[IG(b, a)]$ denotes the expected information gain under the Bayes optimal policy's belief-action marginal distribution and let the belief-action marginal induced by both open-loop and EFE policies have bounded density ratio with the Bayes optimal policy $\left\Vert \frac{d^{\tilde{\pi}}_{P}(b, a)}{d^{\pi}_{P}(b, a)} \right\Vert_{\infty} \leq C$. Under assumptions \ref{assumption:reward_specification} and \ref{assumption:policy_behavior}, the performance gap of the open-loop and EFE policies from the optimal policy are bounded as:
\begin{align}
\begin{split}
    &J_{M}(\pi) - J_{M}(\pi^{open}) \leq \frac{1}{1 - \gamma}\epsilon_{\tilde{\pi}} + \frac{(C + 1)\gamma R_{max}}{(1 - \gamma)^{2}}\epsilon_{IG} \,, \\
    &J_{M}(\pi) - J_{M}(\pi^{EFE}) \leq \frac{1}{1 - \gamma}\epsilon_{\tilde{\pi}} + \frac{(C + 1)\gamma R_{max}}{(1 - \gamma)^{2}}\epsilon_{IG} - \frac{C + 1}{1 - \gamma}\epsilon_{IG} \,.
\end{split}
\end{align}
\end{theorem}

\begin{proof}
Let us start by bounding the absolute value of the EFE policy's performance gap:
\begin{align}
\begin{split}
    &|J_{M}(\pi) - J_{M}(\pi^{EFE})| \\
    &\leq \left|\frac{1}{1 - \gamma}\mathbb{E}_{(b, a) \sim d^{\pi}_{P}}[A^{\pi^{EFE}_{M^{EFE}}}(b, a)]\right| \\
    &\quad + \left|\frac{1}{1 - \gamma}\mathbb{E}_{(b, a) \sim d^{\pi}_{P}}\left[-IG(b, a) + \gamma\left(\mathbb{E}_{b'' \sim P(\cdot|b, a)}[V^{\pi^{EFE}}_{M^{EFE}}(b'')] - \mathbb{E}_{b' \sim P^{open}(\cdot|b, a)}[V^{\pi^{EFE}}_{M^{EFE}}(b')]\right)\right]\right| \\
    &\quad + \left|\frac{1}{1 - \gamma}\mathbb{E}_{(b, a) \sim d^{\pi^{EFE}}_{P}}\left[IG(b, a) + \gamma\left(\mathbb{E}_{b'' \sim P^{open}(\cdot|b, a)}[V^{\pi^{EFE}}_{M^{EFE}}(b'')] - \mathbb{E}_{b' \sim P(\cdot|b, a)}[V^{\pi^{EFE}}_{M^{EFE}}(b')]\right)\right]\right| \,. \\
\end{split}
\end{align}

Examining the second term, we have:
\begin{align}
\begin{split}
    &\left|\frac{1}{1 - \gamma}\mathbb{E}_{(b, a) \sim d^{\pi}_{P}}\left[-IG(b, a) + \gamma\left(\mathbb{E}_{b'' \sim P(\cdot|b, a)}[V^{\pi^{EFE}}_{M^{EFE}}(b'')] - \mathbb{E}_{b' \sim P^{open}(\cdot|b, a)}[V^{\pi^{EFE}}_{M^{EFE}}(b')]\right)\right]\right| \\
    &\leq \left|\frac{1}{1 - \gamma}\mathbb{E}_{(b, a) \sim d^{\pi}_{P}}\left[-IG(b, a) + \frac{\gamma R_{max}}{1 - \gamma}\sqrt{2 IG(b, a)}\right]\right| \\
    &\leq \left|\frac{1}{1 - \gamma}\mathbb{E}_{(b, a) \sim d^{\pi}_{P}}\left[-IG(b, a) + \frac{\gamma R_{max}}{1 - \gamma}IG(b, a)\right]\right| \\
    &= \frac{\gamma R_{max} + \gamma - 1}{(1 - \gamma)^{2}}\left|\mathbb{E}_{(b, a) \sim d^{\pi}_{P}}[IG(b, a)]\right| \\
    &= \frac{\gamma R_{max} + \gamma - 1}{(1 - \gamma)^{2}}\mathbb{E}_{(b, a) \sim d^{\pi}_{P}}[IG(b, a)] \,.
\end{split}
\end{align}

Plugging into the performance gap, we have:
\begin{align}
\begin{split}
    &|J_{M}(\pi) - J_{M}(\pi^{EFE})| \\
    &\leq \left|\frac{1}{1 - \gamma}\mathbb{E}_{(b, a) \sim d^{\pi}_{P}}[A^{\pi^{EFE}_{M^{EFE}}}(b, a)]\right| \\
    &\quad + \frac{\gamma R_{max} + \gamma - 1}{(1 - \gamma)^{2}}\mathbb{E}_{(b, a) \sim d^{\pi}_{P}}[IG(b, a)] + \frac{\gamma R_{max} + \gamma - 1}{(1 - \gamma)^{2}}\mathbb{E}_{(b, a) \sim d^{\pi}_{P}}\left[\left|\frac{d^{\pi^{EFE}}_{P}(b, a)}{d^{\pi}_{P}(b, a)}IG(b, a)\right|\right] \\
    &\leq \left|\frac{1}{1 - \gamma}\mathbb{E}_{(b, a) \sim d^{\pi}_{P}}[A^{\pi^{EFE}_{M^{EFE}}}(b, a)]\right| \\
    &\quad + \frac{\gamma R_{max} + \gamma - 1}{(1 - \gamma)^{2}}\mathbb{E}_{(b, a) \sim d^{\pi}_{P}}[IG(b, a)] + \frac{\gamma R_{max} + \gamma - 1}{(1 - \gamma)^{2}}\left\Vert \frac{d^{\pi^{EFE}}_{P}(b, a)}{d^{\pi}_{P}(b, a)} \right\Vert_{\infty}\mathbb{E}_{(b, a) \sim d^{\pi}_{P}}\left[\left|IG(b, a)\right|\right] \\
    &= \frac{1}{1 - \gamma}\epsilon_{\pi^{EFE}} + \frac{(C + 1)(\gamma R_{max} + \gamma - 1)}{(1 - \gamma)^{2}}\epsilon_{IG} \\
    &\leq \frac{1}{1 - \gamma}\epsilon_{\pi^{open}} + \frac{(C + 1)\gamma R_{max}}{(1 - \gamma)^{2}}\epsilon_{IG} - \frac{C + 1}{1 - \gamma}\epsilon_{IG} \,.
\end{split}
\end{align}

For the open-loop policy which does not have the $IG$ term in the reward, it is easy to see that the performance gap is:
\begin{align}
\begin{split}
    &|J_{M}(\pi) - J_{M}(\pi^{open})| \leq \frac{1}{1 - \gamma}\epsilon_{\pi^{open}} + \frac{(C + 1)\gamma R_{max}}{(1 - \gamma)^{2}}\epsilon_{IG} \,. \\
\end{split}
\end{align}
\end{proof}

\subsection{Proofs for Section \ref{sec:aif_obj_specification}}
\paragraph{Equivalence between state marginal matching \citep{lee2019efficient} and closed-loop EFE in MDP} The marginal state distribution at time step $t$ following policy $\pi$ is defined as:
\begin{align}
\begin{split}
    d^{\pi}_{t}(s_{t}) &= b_{t}(s_{t}|b_{t-1}, \pi) \\
    &= \sum_{s_{t-1}}\sum_{a_{t-1}}P(s_{t}|s_{t-1}, a_{t-1})\pi(a_{t-1}|s_{t-1})b_{t-1}(s_{t-1}) \,,
\end{split}
\end{align}
where the state marginal at one time step depends on the state marginal at the previous time step.

We can define the time-averaged state marginal matching problem as state marginal matching at all time steps:
\begin{align}
\begin{split}
    \arg\min_{\pi} \mathbb{KL}[d^{\pi}(s) || \tilde{P}(s)] &= \arg\min_{\pi}  \frac{1}{T}\sum_{t=0}^{T}\mathbb{KL}[d^{\pi}_{t}(s_{t}) || \tilde{P}(s_{t})] \\
    &= \arg\min_{\pi} \sum_{t=0}^{T}\mathbb{E}_{b_{t}(s_{t}|b_{t-1}, \pi)}[\log b_{t}(s_{t}|b_{t-1}, \pi) - \log \tilde{P}(s_{t})] \,.
\end{split}
\end{align}

Since we consider closed-loop policies where the agent can observe the environment state rather than computing the next state marginal from an imprecise current state marginal, it follows that:
\begin{align}
    b_{t}(s_{t}|b_{t-1}, \pi) = P(s_{t}|s_{t-1}, a_{t-1}) \,.
\end{align}

Thus the objective is equivalent to:
\begin{align}
    \min_{\pi} \mathbb{E}_{P(s_{0:T}, a_{0:T-1})}\left[\sum_{t=0}^{T}\left(\log P(s_{t}|s_{t-1}, a_{t-1}) - \log \tilde{P}(s_{t})\right)\right] \,.
\end{align}
This is the same objective in \citep{lee2019efficient} and also \citep{da2023reward}.

\end{document}